\documentclass[twoside]{article}

\usepackage[accepted]{aistats2019}
\usepackage{hyperref}       % hyperlinks
\usepackage{url}            % simple URL typesetting
\usepackage{subfigure}
\newcommand{\loss}{\ell}
\newcommand{\Loss}{\mathcal{L}}

\newcommand{\D}{\mathcal{D}}
\newcommand{\F}{\mathcal{F}}
\newcommand{\G}{\mathcal{G}}
\newcommand{\Dtr}{D_{\mathrm{tr}}}
\newcommand{\ntr}{n_{\mathrm{tr}}}
\newcommand{\Mtr}{M_{\mathrm{tr}}}
\newcommand{\Str}{S_{\mathrm{tr}}}
\newcommand{\Xtr}{X_{\mathrm{tr}}}
\newcommand{\Ltr}{L_{\mathrm{tr}}}
\newcommand{\tOtr}{\tilde{\Omega}_{\mathrm{tr}}}
\newcommand{\Err}{\mathcal{E}}
\newcommand{\bErr}{\bar{\mathcal{E}}}
\newcommand{\hErr}{\hat{\mathcal{E}}}

\newcommand{\uni}{\mathcal{U}}
\usepackage{graphicx}
\usepackage{float}
\usepackage{amsmath}
\usepackage{amsthm}
\usepackage{amsfonts}
\usepackage{hyperref}
\usepackage{natbib}
\usepackage{algorithm}
\usepackage{algorithmic}
\usepackage{wrapfig}
\allowdisplaybreaks

\theoremstyle{plain}
\newtheorem{theorem}{Theorem}
\newtheorem{proposition}[theorem]{Proposition}
\newtheorem{lemma}[theorem]{Lemma}

\theoremstyle{definition}
\newtheorem{definition}[theorem]{Definition}

\theoremstyle{remark}

\newcommand{\iset}[1]{\left[#1\right]}

\newcommand{\rar}{\rightarrow}
\newcommand{\lar}{\leftarrow}

\newcommand{\norm}[1]{\left\Vert #1 \right\Vert}
\newcommand{\Prb}[1]{\Pr\left( #1 \right)}
\newcommand{\E}[1]{\mathbb{E}\left[#1\right]}
\newcommand{\EE}[2]{\mathbb{E}_{#1}\left[#2\right]}
\newcommand{\I}[1]{\mathbb{I}\left\{#1\right\}}
\newcommand{\sign}[1]{\mathrm{sign}\left( #1 \right)}
\newcommand{\addeq}{\addtocounter{equation}{1}\tag{\theequation}}

\newcommand{\real}{\mathbb{R}}
\newcommand{\normal}{\mathcal{N}}

\allowdisplaybreaks
\usepackage{enumitem}
\setlist[itemize]{leftmargin=*}

\begin{document}

\runningtitle{Generalization Bias of Two Layer Convolutional Linear Classifiers}


\twocolumn[

\aistatstitle{Towards Understanding the Generalization Bias of Two Layer Convolutional Linear Classifiers with Gradient Descent}

\aistatsauthor{ Yifan Wu\And Barnab\'as P\'oczos \And  Aarti Singh }

\aistatsaddress{ Carnegie Mellon University \\yw4@cs.cmu.edu \And Carnegie Mellon University \\bapoczos@cs.cmu.edu \And Carnegie Mellon University \\ aarti@cs.cmu.edu} ]

\begin{abstract}
A major challenge in understanding the generalization of deep learning is to explain why (stochastic) gradient descent can exploit the network architecture to find solutions that have good generalization performance when using high capacity models. We find simple but realistic examples showing that this phenomenon exists even when learning linear classifiers --- between two linear networks with the same capacity, the one with a convolutional layer can generalize better than the other when the data distribution has some underlying spatial structure. We argue that this difference results from a combination of the convolution architecture, data distribution and gradient descent, all of which are necessary to be included in a meaningful analysis. We analyze of the generalization performance as a function of data distribution and convolutional filter size, given gradient descent as the optimization algorithm, then interpret the results using concrete examples. Experimental results show that our analysis is able to explain what happens in our introduced examples.
\end{abstract}

\section{Introduction}
\label{sec:intro}

It has been shown that the capacities of successful deep neural networks are typically large enough such that %there exist certain weights that can fit the training data well but the network still generalizes poorly
they can fit random labelling of the inputs in a dataset \citep{zhang2016understanding}. Hence an important problem is to understand why gradient descent (and its variants) is able to find the solutions that generalize well on unseen data. 
%Recently an increasing amount of work has been done in this direction (cite some papers). 
Another key factor, besides gradient descent, in achieving good generalization performance in deep neural networks is architecture design with weight sharing (e.g. Convolutional Neural Networks (CNNs) \citep{lecun1998gradient} and Long Short Term Memories (LSTMs) \citep{hochreiter1997long} ). To the best of our knowledge, none of the existing work on analyzing the generalization bias of gradient descent takes these specific architectures into formal analysis. One may conjecture that the advantage of weight sharing is caused by reducing the network capacity compared with using fully connected layers
%, according to some prior knowledge on the data distribution, 
without talking about gradient descent. However, as we will show later, there is a joint effect between network architectures and gradient descent on the generalization performance even if the model capacity remains unchanged. In this work we try to analyze the generalization bias of two layer CNNs together with gradient descent, as one of the initial steps towards understanding the generalization performance of deep learning in practice.

CNNs have proven to be successful in learning tasks where the data distribution has some underlying spatial structure such as image classification \citep{krizhevsky2012imagenet, he2016deep}, Atari games \citep{mnih2013playing} and Go \citep{silver2017mastering}. A common view of how CNNs work is that convolutional filters extract high level features while pooling exploits spatial translation invariance \citep{Goodfellow-et-al-2016}. Pooling, however, is not always used, especially in reinforcement learning (RL) tasks (see the networks used in Atari games \citep{mnih2013playing}, and Go \citep{silver2017mastering}) even if exploiting some level of spatial invariance is desired for good generalization. For example, if we are training a robot arm to pick up an apple from a table, one thing we are expecting is that the robot learns to move its arm to the left if the apple is on its left and vice versa. If we use a policy network to decide ``left" or ``right", we expect the network to be able to generalize without being trained with all of the pixel level combinations of the (arm, apple) location pair. In order to see whether stacking up convolutional filters and fully connected layers without pooling can still exploit the spatial invariance in the data distribution, we design the following tasks, which are the simplified 1-D version of 2-D image based classification and control tasks:
\vspace{-0.3cm}
\begin{itemize}
\item{\textbf{Binary classification (Task-Cls):} Suppose we are trying to classify object A v.s. B given a 1-D $d$-pixel ``image" as input. We assume that only one of the two objects appears on each image and the object occupies exactly one pixel. In pixel level inputs $x\in \{-1, 0, +1\}^d$, we use $+1$ to represent object A, $-1$ for object B and $0$ for nothing. We use label $y=+1$ for object A and $y=-1$ for object B. The resulting dataset looks as follows:
\begin{align*}
 x = [0,......,0, -1, 0,...,0] \rar y = -1 \,;\\
 x = [0,...,0,+1,0,......,0] \rar y = +1 \,.
\end{align*}
The entire possible dataset contains $2d$ samples.
}
\item{\textbf{First-person vision-based control (Task-1stCtrl):} 
Suppose we are doing first-person view control in 3-D environments with visual images as input, e.g. robot navigation, and the task is to go to the proximity of object A. One decision the robot has to make is to turn left if the object is on the left half of the image and turn right if it is on the right half. We consider the simplified 1-D version, where each input $x\in \{0,1\}^d$ contains only one non-zero element $x_i=1$ with $y=-1$ if the object is on the left half ($i\le d/2$) and $y=+1$ if the object is on the right half ($i>d/2$). The resulting dataset looks as follows:
\begin{align*}
x = [0,...1..,0,......,0] \rar y = -1 \,;\\
x = [0,......,0,...1..,0] \rar y = +1 \,.
\end{align*}
The entire possible dataset contains $d$ samples.
}
\item{\textbf{Third-person vision-based control (Task-3rdCtrl):} We consider the fixed third-person view control, e.g. controlling a robot arm, and the task is to control the agent (arm), denoted by object B and represented by $-1$, to touch the target object A, represented by $+1$, in the scene. Again we want to move the arm B to the left ($y=-1$) if A is on the left of B and move it to the right ($y=+1$) if A is on the right. The resulting dataset looks as follows:
\begin{align*}
x = [0,...+1....-1.....,0] \rar y = -1  \,;\\
x = [0,......-1...+1...,0] \rar y = +1 \,.
\end{align*} 
The entire possible dataset contains $d(d-1)$ samples.
}
\end{itemize}
\vspace{-0.3cm}
Although all of the three tasks we described above have a finite number of samples in the whole dataset we still seek good generalization performance on these tasks when learning from only a subset of them. We do not want the learner to see almost all possible pixel-wise appearance of the objects before it is able to perform well. Otherwise the sample complexity will be huge when the resolution of the image becomes higher and the number of objects involved in the task grows larger.

One key property of all these three tasks we designed is that the data distribution is linearly separable even without introducing the bias term. That is, for each of the tasks, there exist at least one $w\in \real^d$ such that $y=\sign{w^Tx}$ for all $(x, y)$ in the whole dataset. This property gives superior convenience in both experiment control and theoretical analysis,  while several important aspects
, as we will explain later in this section, 
in analyzing the generalization of deep networks are still preserved: The linear separator $w$ for a training set is not unique and the interesting question is why different algorithms (architecture plus optimization routine) can find solutions that generalize better or worse on unseen samples.
%We will provide some intuition in Appendix.
%\begin{wrapfigure}{R}{0.45\textwidth}
%\vspace{-0.4cm}
%\begin{minipage}{0.44\textwidth}
%\vspace{-0.3cm}
\begin{figure}[t]
\vskip -0.0in
%\begin{center}
\centerline{\includegraphics[width=0.9\columnwidth]{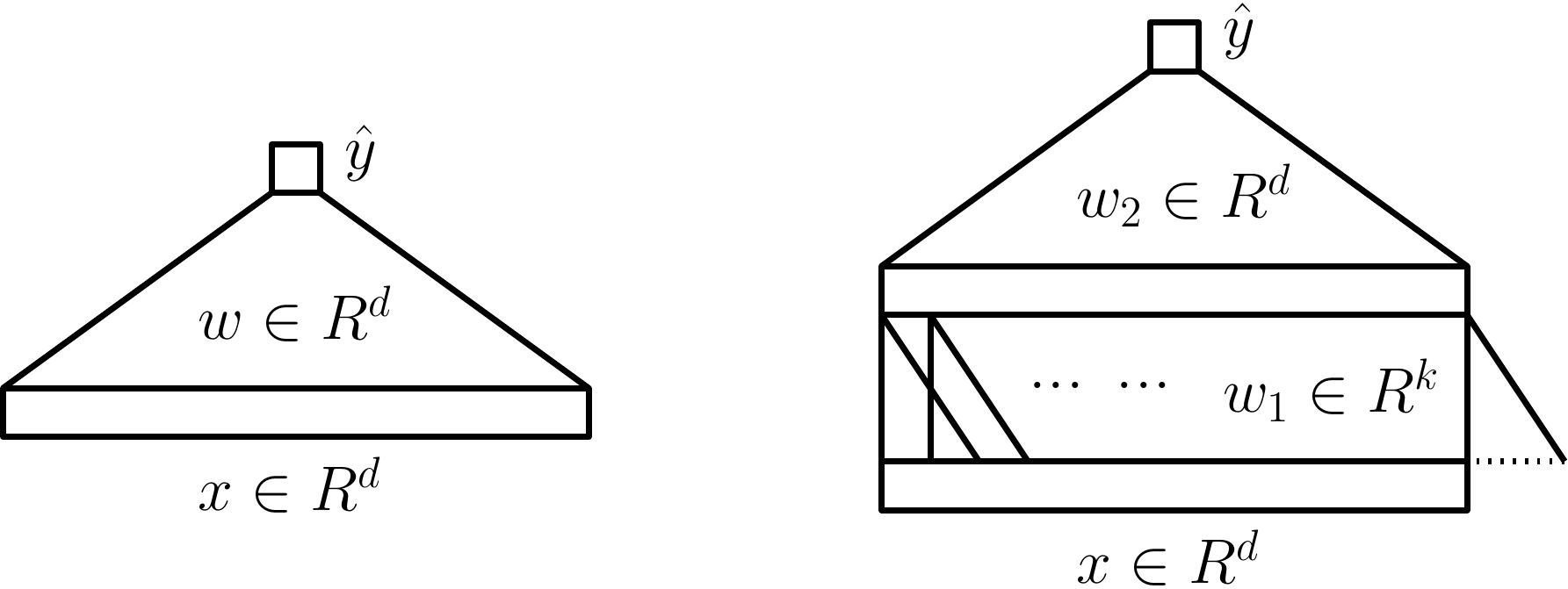}}
\caption{ Left: $\hat{y}=\sign{w^Tx}$ --- \textbf{Model-1-Layer}.   Right: $\hat{y}=\sign{ w_2^T\mathrm{Conv}(w_1, x)}$ --- \textbf{Model-Conv-$k$}. }
\label{fig:net-compare}
%\end{center}
\vskip -0.0in
\end{figure}
%\end{minipage}
%\vspace{-0.3cm}
%\end{wrapfigure}
%\vspace{-0.3cm}  

Since the data distribution is linearly separable the immediate learning algorithm one would try on these tasks is to train a single layer linear classifier using logistic regression, SVM, etc. However, this seems to be not exploiting the spatial structure of the data distribution and we are interested in seeing whether adding convolutional layers could help. More specifically, we consider adding a convolution layer between the input $x$ and the output layer, where the convolution layer contains only one size-$k$ filter ($\mathrm{output\_channel}=1$) with $\mathrm{stride}=1$ and without non-linear activation functions. Figure~\ref{fig:net-compare} shows the two models we are comparing. We also experimented with the fully-connected two layer linear classifier where the first layer has weight matrix $W_1\in \real^{d\times d}$ without non-linear activation. This model has the same generalization behavior as a single-layer linear classifier in all of our experiments so we will not show these results separately. This phenomenon, however, may also exhibit an interesting problem to study. 

It is worth noting that Model-1-Layer and Model-Conv-$k$ represent exactly the same set of functions. That is, for any $w$ in Model-1-Layer we are able to find $(w_1, w_2)$ in Model-Conv-$k$ such that they represent the same function, and vice versa. Therefore, both of the two models have the same capacity and any difference in the generalization performance cannot be explained by the difference of capacity. We compare the generalization performance of the two models in Figure~\ref{fig:net-compare} on all of the three tasks we have introduced. 
%Somewhat surprisingly, 
As shown in Figure~\ref{fig:gcomp}, Model-Conv-$k$ outperforms Model-$1$-Layer on all of the three tasks. The rest of our paper is motivated by explaining the generalization behavior of Model-Conv-$k$.
%\begin{wrapfigure}{R}{0.6\textwidth}
%\vspace{-0.4cm}
%\begin{minipage}{0.59\textwidth}
\begin{figure}[t]
\centering
\subfigure[Task-Cls]{
\includegraphics[width=0.43\columnwidth]{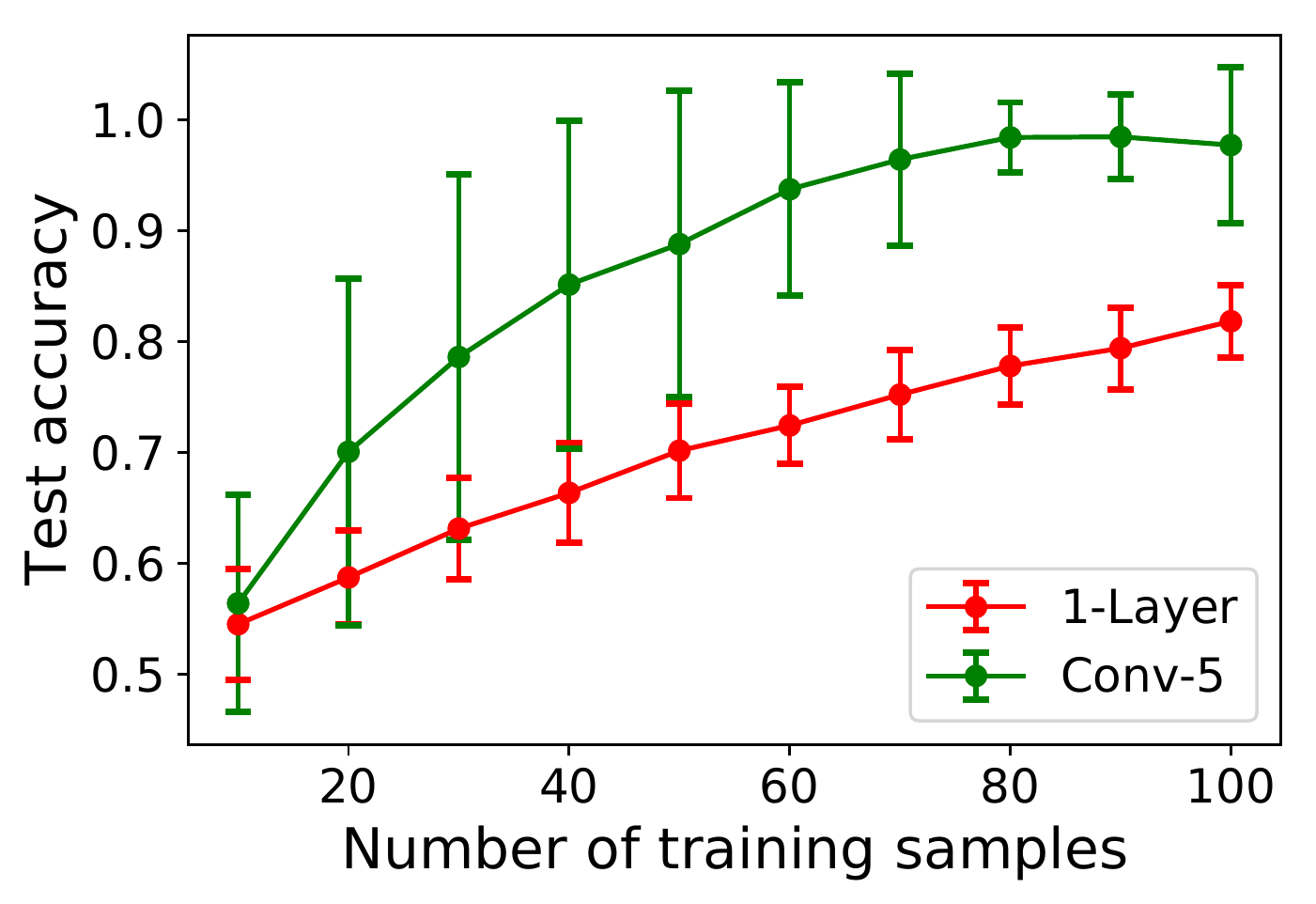}
\label{fig:gcomp1}
}
\subfigure[Task-1stCtrl]{
\includegraphics[width=0.43\columnwidth]{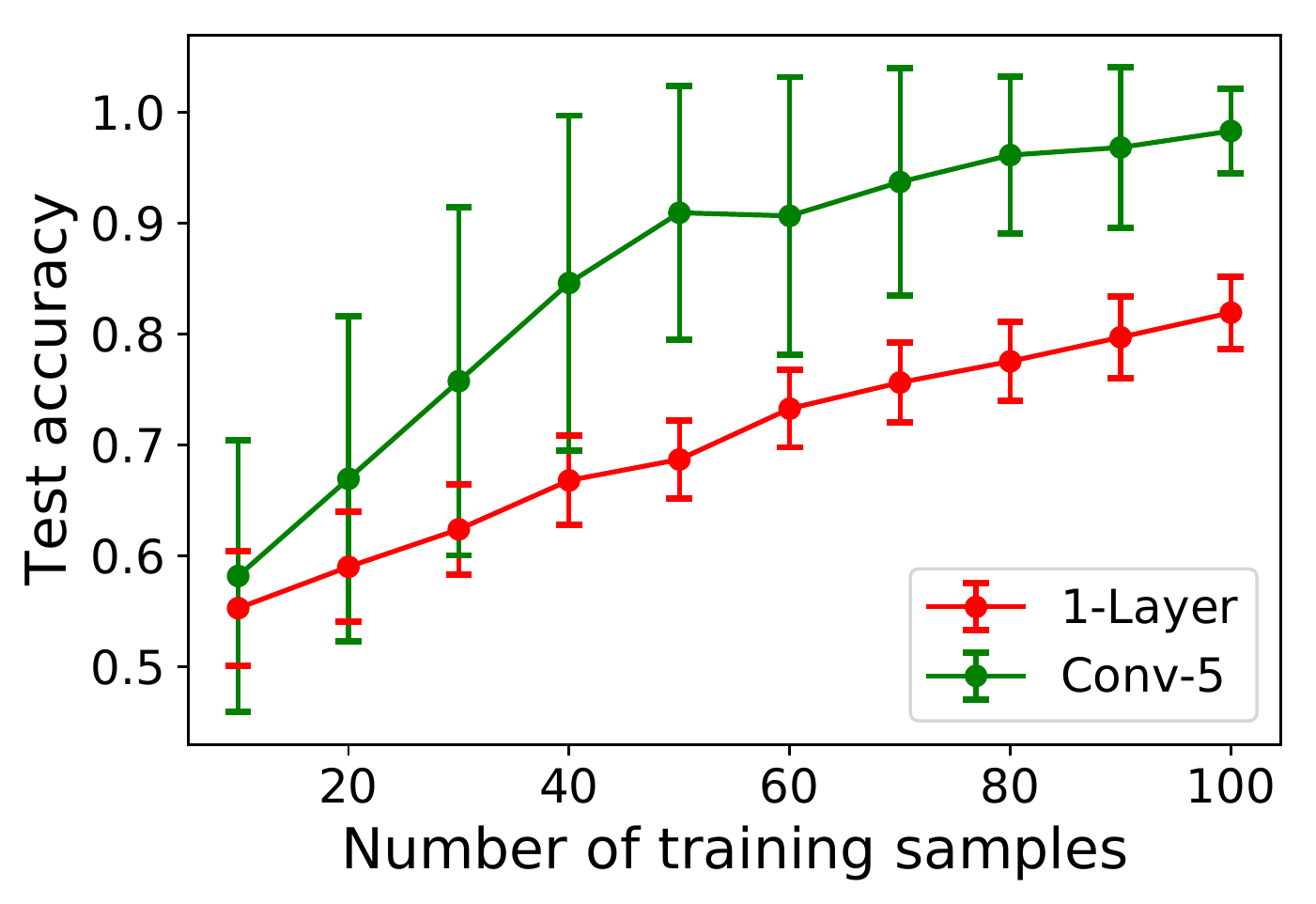}
\label{fig:gcomp2}
}
\subfigure[Task-3rdCtrl]{
\includegraphics[width=0.43\columnwidth]{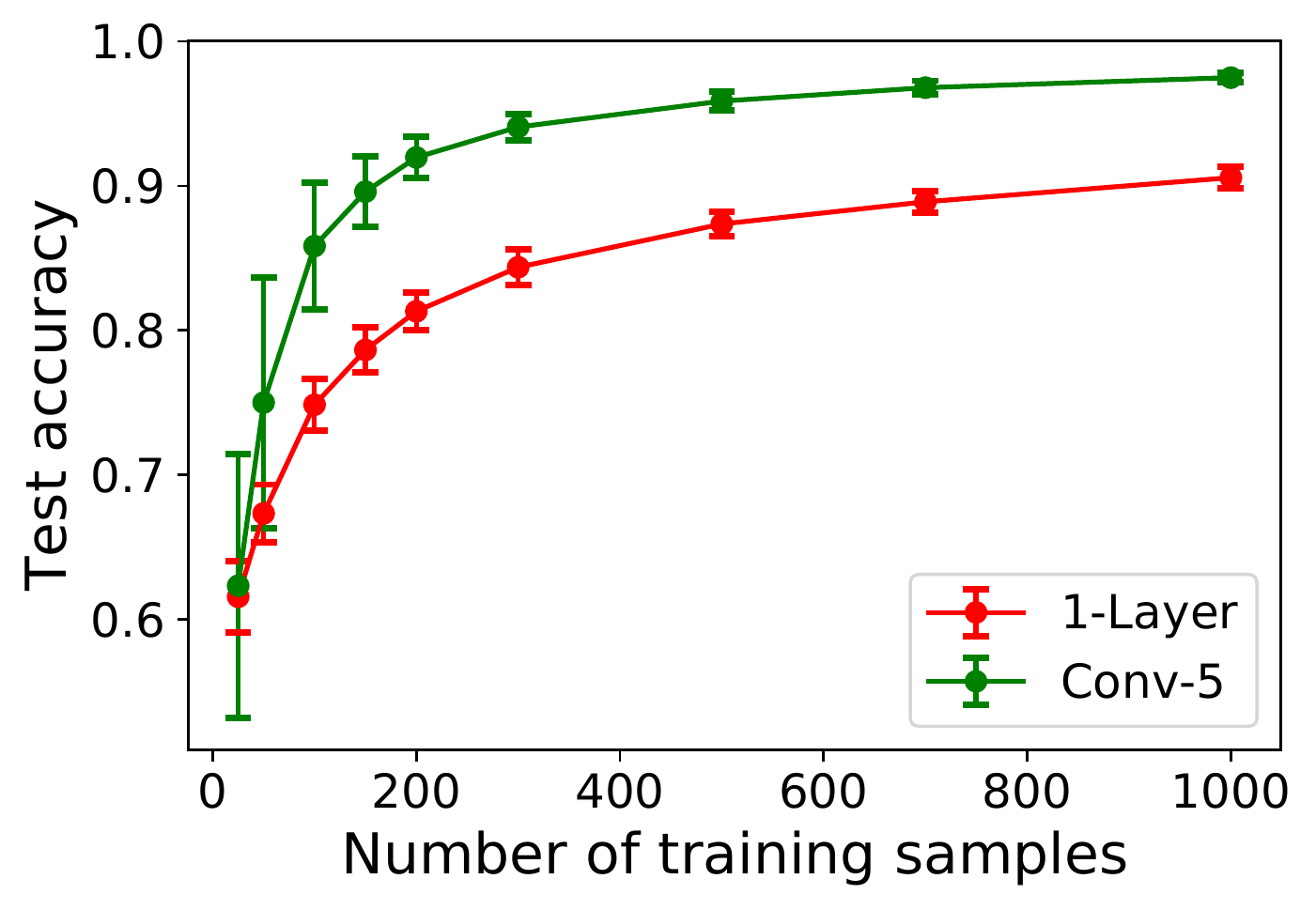}
\label{fig:gcomp3}
}
\vspace{-0.4cm}
\caption{
Comparing the generalization performance between single layer and two layer convolutional linear classifiers with different sizes of training data. Training samples are uniformly sampled from the whole dataset ($d=100$) with replacement. The models are trained by minimizing the hinge loss $(1-yf(x))_+$ using full-batch gradient descent. Training stops when training reaches $0$. Trained models are then evaluated on the whole dataset (including the training samples). For convolution layer we use a single size-$5$ filter with stride $1$ and padding with $0$. Each plotted point is based on repeating the same configuration for 100 times.  
}
\label{fig:gcomp}
\end{figure}
%\end{minipage}
%\vspace{-0.3cm}

%\end{wrapfigure}
% e.g. (i) Why it is better than Model-$1$-Layer, i.e. why the overparameterization is helpful? (ii) Why the difference is small when we have very few training samples? (iii) Where does the high variance come from?
Explaining our empirical observations requires a generalization analysis that depends on data distribution, convolution structure and gradient descent. Any of these three factors cannot be isolated from the analysis for the following reasons: 
%\vspace{-0.3cm}
\begin{itemize}
\item{\textbf{Data distribution:} If we randomly flip the label for each data point independently then all models will have the same generalize performance on unseen samples.  
}
\item{\textbf{Convolution structure:} The network structure is the main factor that we are trying to analyze. We further argue that we should explain the advantage of convolution and not just depth. This is because, as we mentioned earlier, adding a fully connected layer does not provide any advantage compared with a single-layer model in all of our experiments.  
}
\item{\textbf{Gradient descent:} The analysis should also include the optimization algorithm since the function classes represented by the two models we are comparing are equivalent. For example, the Model-Conv-$k$ can be optimized in the way that we first find a solution $w$ by optimizing Model-1-Layer then let $w_1=[1,0,...,0]$ and $w_2=w$, which also gives a solution for Model-Conv-$k$ but has no generalization advantage. Therefore, analyzing how gradient descent is able to exploit the convolution structure 
%to find a solution amongst many equivalent solutions with same training error but better generalization error
is necessary to explain the generalization advantage in our experiments. 
}
\end{itemize}

In this paper we provide a data dependent analysis on the generalization performance of two layer convolutional linear classifiers given gradient descent as the optimizer. In Section~\ref{sec:ana} we first give a general analysis then interpret our results using specific examples. In Section~\ref{sec:exp} we empirically verify that our analysis is able to explain the observations in our experiments in Figure~\ref{fig:gcomp}. %In Section~\ref{sec:dis} we further discuss some limitations of our analysis and possible future directions. 
Due to space constraints, proofs are relegated to the appendix.

Our main contribution can be highlighted as follows: (i) We design simple but realistic examples that are theory-friendly while preserving important challenges in understanding what is happening in practice. (ii) We are the first to provide a formal generalization analysis that considers the interaction among data distribution, convolution, and gradient descent, which is necessary to provide meaningful understanding for deep networks.  (iii) We derive a closed form weight dynamics under gradient descent with a modified hinge loss and relate the generalization performance to the first singular vector pair of some matrix computed from the training samples. (iv) We interpret the results with one of our concrete examples using Perron-Frobenius Theorem for non-negative matrices, which shows how much sample complexity we can save by adding a convolution layer. (v) Our result reveals an interesting difference between the generalization bias of ConvNet and that of traditional regularizations --- The bias itself requires some training samples to be built up. (vi) Our experiments show that our analysis is able to explain what happens in our examples. More specifically, we show that the performance under our modified hinge loss is strongly correlated with the performance under the real hinge loss.

%Motivation of the analysis. Used in practice. Models to compare. Simple and realistic examples. Property of these tasks. Observations: Capture invariance without pooling. No convolution, no better. Need for analysis: data distribution - network structure - gradient descent. All of these must be included. Closed form weight dynamics. A sufficient condition on when conv layer helps generalization. Verified by experiments. Interpreting the condition. Interesting discussion: margin/bias/early stop helps/hurts convergence/generalization. Not only local invariance - inv = 100. sufficient but not necessary condition, inv = 10. The effect of random initialization. Why the real performance is a complicated problem. Related work.  
\vspace{-0.2cm}
\section{Preliminaries}
\label{sec:pre}
\vspace{-0.2cm}

%In this section, we formulate the problem we studied and introduce some background for later analysis. 

\subsection{Learning Binary Classifiers}

We consider learning binary classifiers $\hat{y} = \sign{f_w(x)}$  with a function class $f$ parameterized by $w$, in order to predict the real label $y\in \{-1,+1\}$ given an input $x\in \real^d$. A random label is predicted with equal chance if $f_w(x)=0$. %\todoy{need to be clear and careful. assuming $1/2$ risk when $f_w(x)\rar 0$ may seem abnormal but is useful for analysis. One way to work around this is to make a continuous prob. prediction. Then make the transition period sharp to the limit.} 
In a single layer linear classifier (Model-1-Layer) we have $w\in \real^d$ and $f_w(x) = w^Tx$. In the two layer convolutional linear classifiers (Model-Conv-$k$) we have $w=(w_1, w_2)$ where the convolution filter $w_1\in \real^k$ and the output layer $w_2\in \real^d$ ($k\le d$). $f_w$ can be written as
$f_w(x) = \sum_{i=1}^d w_{2,i} \sum_{j=1}^{k} w_{1,j} x_{i+j-1}$
where every term whose index is out of range is treated as zero.

We denote the entire data distribution as $\D$, which one can sample data points $(x, y)$s from. We say drawing a training set $D\sim \D$ when we independently sample $n=|D|$ data points from $\D$ with replacement and take the collection as $D$. For finite datasets, e.g. in the tasks we introduced, we assume the data distribution is uniform over all data points. In this paper we only consider the case where there is no noise in the label, i.e. the true label $y$ is always deterministic given an input $x$, so that we can write $(x,y)\in D$ or $x\in D$ interchangeably. %For the convenience of analysis, we further assume $x=\vec{0}$ does not belong to the support of $\D$ (or has measure $0$). \todoy{This may not be the correct assumption}    

Given a training set $\Dtr \sim \D$ with $\ntr$ samples and a model $f_w$, we learn the classifier by minimizing the empirical hinge loss %\footnote{Here we generalize the typical hinge loss $(1-yf_w(x))_+$ to be $(\gamma-yf_w(x))_+$ with $\gamma>0$ for later convenience.}
$\Loss(w; \Dtr) = \frac{1}{\ntr} \sum_{(x,y)\in \Dtr} (1-yf_w(x))_+$
using full-batch gradient descent with learning rate $\alpha>0$:
\begin{align*}
w^{t+1} & = w^t - \alpha \nabla_w \Loss(w^t; \Dtr) \\
& = w^t + \frac{\alpha}{\ntr}\sum_{(x,y)\in \Dtr} \I{yf_{w^t}(x)<1} y \nabla_w f_{w^t}(x) \,.
\end{align*}
%In our experiments where the data distributions are linearly separable and $f_w(x)$ are linear models, an optimal solution $w^*$ can always minimize the training loss to be zero.

Given a classifier $f_w$ and data distribution $\D$, the generalization error can be written as 
\begin{align*}
&\Err(w;\D)  = \EE{\D}{\EE{\hat{y}}{\I{\hat{y}\ne y}}} \\
&= \EE{ \D}{\I{yf_w(x)<0}+\frac{1}{2}\I{yf_w(x)=0}}\\
&  = \EE{\D}{\bErr(yf_w(x))} \,, \addeq\label{eq:raw-err}
\end{align*}
where we define function $\bErr:\real \mapsto \{0, \frac{1}{2}, 1\}$ as $
\bErr(x) = \I{x<0}+\frac{1}{2}\I{x=0} 
$
which is non-increasing and satisfies $\bErr(\alpha x) = \bErr(x)$ for any $\alpha>0$.

\subsection{An Alternative Form for Two Layer ConvNets} 

For the convenience of analysis, we use an alternative form to express Model-Conv-$k$. Let $A_x \in \real^{d\times k}$ be $[x, x_{\lar_1},..., x_{\lar_{k-1}}]$, where $k$ is the size of the filter and $x_{\lar_l}$ is defined as the input vector left-shifted by $l$ positions: $x_{\lar_l, i} = x_{i+l}$ (pad with 0 if out of range).
Then $f_w(x)$ can be written as $f_w(x) = w_1^T A_x^T w_2$. The definition of $A_x$ is visualized in Figure~\ref{fig:A-matrix}.

%\begin{wrapfigure}{R}{0.45\textwidth}
%\vspace{-0.4cm}
%\begin{minipage}{0.44\textwidth}

\begin{figure}
\vskip -0.1in
%\begin{center}
\centerline{\includegraphics[width=0.5\columnwidth]{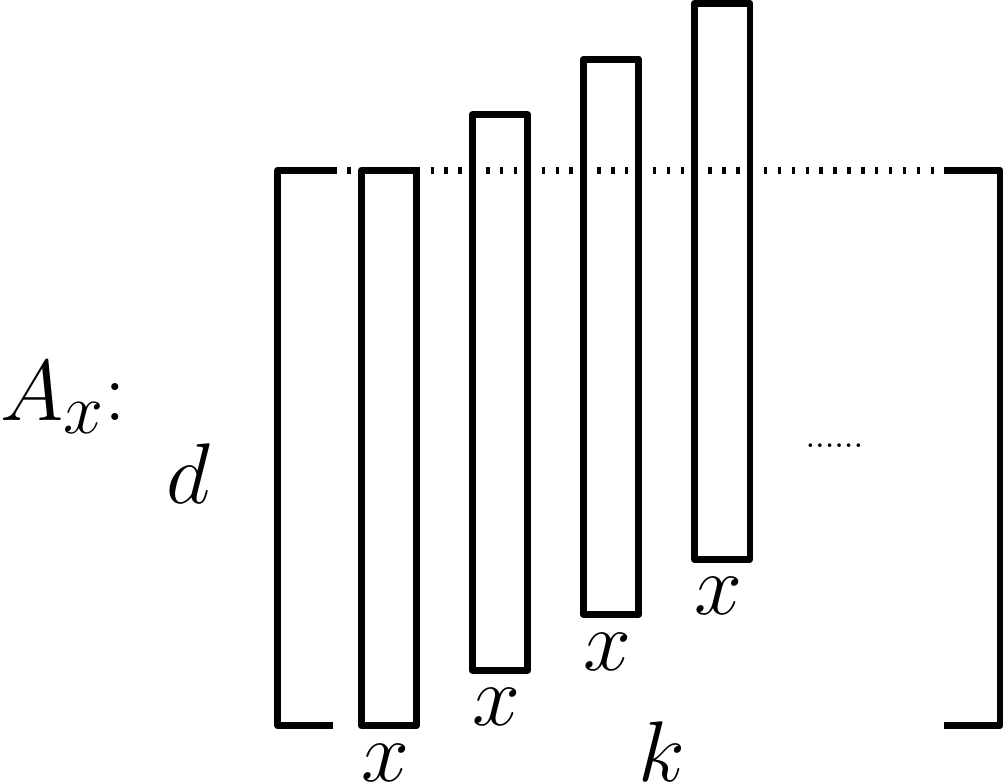}}
\caption{Matrix $A_x\in \real^{d\times k}$ given $x\in \real^d$.}
\label{fig:A-matrix}
%\end{center}
%\vskip -0.5in
\end{figure}
%\end{minipage}
%\vspace{-1.5cm}
%\end{wrapfigure}

Further define $M_{x,y} = yA_x$ then we have $yf_w(x)=w_1^T M_{x,y}^T w_2$. We can write the empirical loss as
$
\Loss(w; \Dtr)
= \frac{1}{\ntr} \sum_{(x,y)\in \Dtr} (1-w_1^T M_{x,y}^T w_2)_+ 
$
and the generalization error as
$
\Err(w;\D) = \EE{(x,y)\sim \D}{\bErr\left(w_1^T M_{x,y}^T w_2\right)}$.
\section{Theoretical Analysis}
\label{sec:ana}

In this section we analyze the generalization behavior of Model-Conv-$k$ when training with gradient descent. We introduce a modified version of the hinge loss which enables a closed-form  expression for the weight dynamics .  Based on the closed-form we show that the weights converge to some specific directions as $t\rar\infty$. Plugging the asymptotic weights back to the generalization error gives the observation that the generalization performance depends on the first singular vector pair of the average $M_{x,y}$ over the training samples. We interpret our result under Task-Cls, which shows that our analysis is well aligned with the empirical observations and quantifies how much ($\approx 2k-1$ times) sample complexity  can be saved by adding a convolution layer. 

\subsection{The Extreme Hinge Loss (X-Hinge)} 

We consider minimizing a linear variant of the hinge loss $\loss(w;x,y) = -yf_w(x)$. We call it \emph{the extreme hinge loss} because the gradient of this loss is the same as the gradient of the hinge loss $\loss(w;x,y) = (1-yf_w(x))_+$ when $yf_w(x)<1$. Then the training loss becomes
$\Loss(w; \Dtr)= -  w_1^T \Mtr^T w_2$
%& = - \frac{1}{\ntr} \sum_{(x,y)\in \Dtr} w_1^T M_{x,y}^T w_2 = -  w_1^T \Mtr^T w_2$
where we define $\Mtr = \frac{1}{\ntr} \sum_{(x,y)\in D} M_{x,y} \in \real^{d\times k}$. Note that minimizing this loss will lead to $\Loss \rar -\infty$. However, since the normal hinge loss can be viewed as a fit-then-stop version of X-hinge, considering loss $\Loss \rar -\infty$ in our cases gives interesting insights about the generalization bias of Conv-$k$ under the normal hinge loss. In the next section we will further verify the correlation between X-hinge and the normal hinge loss through experiments, which can be summarized as follows:  

\emph{(i) Under X-hinge $(w_1, w_2)$ converges to a limit direction which brings superior generalization advantage. (ii) Conv-$k$ generalizes better under normal hinge because the weights tend to converge to this limit direction (but stopped when training loss reaches $0$). (iii) The variance (due to different initialization)  in the generalization performance of Conv-$k$ comes from how close the weights are to this limit direction when training stops. }

\subsection{An Asymptotic Analysis}

The full-batch gradient descent update for minimizing X-hinge with learning rate $\alpha$ is $w_1^{t+1} = w_1^t + \alpha \Mtr^T w_2^t$ and 
$w_2^{t+1} = w_2^t + \alpha \Mtr w_1^t$. For the simplicity of writing our analysis we let $w_2^0 = 0$, which does not affect our theoretical conclusion. Now we try to analyze the generalization error when $t\rar \infty$. (See Appendix for a finite-time closed form expression of $w^t$.) First we will show that the weight converge to a specific direction as $t\rar \infty$ given fixed $w_1^0$:
\begin{lemma}
\label{lemma:w-asym}
For any training set $\Dtr$ let $\Mtr=U\Sigma V^T$ be (any of) its SVD and $\sigma_1\ge \sigma_2 \ge ... \ge \sigma_k \ge 0$ be the diagonal of $\Sigma$ with $\sigma_1>0$.\footnote{We implicitly assume that the data distribution $\D$ satisfies $\Prb{\Mtr=0}=0$ for any $\ntr>0$, which is true in all of our examples.}
 Denote $1\le m\le k$ be the largest number such that $\sigma_1=\sigma_m$, then we have
\begin{align*}
 w_1^\infty \doteq  \lim_{t\rar +\infty} \frac{2 w_1^t}{(1+\alpha \sigma_1)^t} = V_{:m} V_{:m}^T w_1^0 \,,\\ w_2^\infty \doteq  \lim_{t\rar +\infty} \frac{2 w_2^t}{(1+\alpha \sigma_1)^t} = U_{:m} V_{:m}^T w_1^0\,, \addeq\label{eq:weights-limit}
\end{align*}
where $A_{:m}$ denotes the first $m$ columns of a matrix $A$. 
\end{lemma}

Let $w^\infty = (w_1^\infty, w_2^\infty)$ be a random variable that depends on $(\Dtr, w_1^0)$ and $\F^\infty(x,y,w_1^0,\Dtr)  \doteq yf_{w^\infty}(x) = {w_1^\infty}^TM_{x,y}^T w_2^\infty  = {w_1^0}^TV_{:m}V_{:m}^TM_{x,y}^TU_{:m}V_{:m}^Tw_1^0$.
We define the \emph{asymptotic generalization error} for Model-Conv-$k$ with gradient descent on data distribution $\D$ as
\footnote{Note that $\Err(w^\infty, \D)=\lim_{t\rar\infty} \Err(w^t, \D)$ may not hold due to the discontinuity of $\I{\cdot}$.} 
\begin{align*}
\Err^\infty_{\mathrm{Convk}}(\D)&  \doteq \EE{\Dtr, w_1^0}{\Err(w^\infty, \D)} \\
& = \EE{w_1^0, \Dtr, (x,y)}{\bErr\left(\F^\infty(x,y,w_1^0,\Dtr)\right)} \,.\addeq\label{eq:asym-err-1}
\end{align*}
One can further remove the dependence on $w_1^0$ when using Gaussian initialization:
\begin{theorem}
\label{thm:asym-err-bound}
Consider training Model-Conv-$k$ by gradient descent with initialization $w_1^0 \sim \normal(0, b^2I_k)$ for some $b>0$ and $w_2^0=0$. Let $UV_1^{M}$ denote the set of left-right singular vector pairs corresponding to the largest singular value $\sigma_1$ for a given matrix $M$. The asymptotic generalization error in \eqref{eq:asym-err-1} can be upper bounded by
%\begin{align*}
$
\Err^\infty_{\mathrm{Convk}}(\D) 
 \le \EE{\Dtr, (x,y)}{\bErr\left(\min_{(u,v)\in UV_1^{\Mtr}} v^T M_{x,y}^T u \right) }$.% \,.\addeq\label{eq:asym-err-2}
%\end{align*}
\end{theorem}
When the first singular vector pair of $\Mtr$ is unique (which is always true when $\ntr$ is not too small in our experiments), denoted by $(u, v)$, we have $m=1$ and Lemma~\ref{lemma:w-asym} says that $w_1^t$ converges to the same direction as $v$ while $w_2^t$ converges to the same direction as $u$. In this case Theorem~\ref{thm:asym-err-bound} holds with equality and  we can remove the $\min$ operator %in \eqref{eq:asym-err-2}
. The asymptotic generalization performance is characterized by how many data points in the whole dataset can be correctly classified by Model-Conv-$k$ with the first singular vector pair of $\Mtr$ as its weights. Later on we will show that this quantity is highly correlated with the real generalization performance in practice where we use the original hinge loss but not the extreme one.       

\subsection{Interpreting the Result with Task-Cls}

We will use our previously introduced task Task-Cls to show that the quantity in Theorem~\ref{thm:asym-err-bound} is non-vacuous: it saves approximately $2k-1$ times samples over Model-$1$-Layer in Task-Cls.

\subsubsection{Decomposing the generalization error}

\textbf{Notation.} For any $l\in \iset{d}=\{1,...,d\}$ define $e_l \in \{0,1\}^d$ to be the vector that has $1$ in its $l$-th position and $0$ elsewhere. Then the set of inputs $x$ in Task-Cls is the set of $e_l$ and $-e_l$ for all $l$.     
Note that in Task-Cls $y_{(-x)}=-y_{(x)}$ and $M_{-x,-y}=M_{x,y}$ so each pair of data points $e_l$ and $-e_l$ can be treated equivalently during training and test. Thus we can think of sampling from $\D$ as sampling from the $d$ positions. Let $\uni\iset{d}$ denote the uniform distribution over $\iset{d}$. Given a training set $\Dtr$ define $\Str=\{l\in \iset{d}\,:\,e_l\in \Dtr \lor -e_l\in\Dtr\}$ to be the set of non-zero positions that appear in $\Dtr$.

To analyze the quantity in Theorem~\ref{thm:asym-err-bound} we notice that all elements in $\Mtr$ are non-negative for any $\Dtr$. By applying the Perron-Frobenius theorem \citep{frobenius1912matrizen} which characterizes the spectral property for non-negative matrices we can further decompose it %\eqref{eq:asym-err-2} 
into two parts. We first introduce the following definition\footnote{See appendix for what a primitive matrix looks like and what it indicates.}:
\begin{definition}
Let $A\in \real^{k\times k}$ be a non-negative square matrix. $A$ is \emph{primitive} if there exists a positive integer $t$ such that $A^t_{ij}>0$ for all $i,j$. 
\end{definition}

Now we are ready to state the following theorem:
\begin{theorem}
Let $\Omega(A)$ be the event that $A$ is primitive and $\Omega^c(A)$ be its complement. Consider training Model-Conv-$k$ with gradient descent on Task-Cls. The asymptotic generalization error defined in \eqref{eq:asym-err-1} can be upper bounded by 
%\begin{align*}
$\Err^\infty_{\mathrm{Convk}}  \le \Prb{\Omega^c(\Mtr^T\Mtr)}+ \frac{1}{2}\EE{l\sim \uni\iset{d} }{\Prb{\forall l'\in \Str, |l'-l|\ge k}} $.
%\,.\addeq\label{eq:err-decomp}
%\end{align*}
\label{thm:err-decomp}
\end{theorem} 
\if0
\emph{Proof sketch.} The proof is mainly based on the following facts: (i) If $\Mtr^T\Mtr$ is primitive then the first pair of singular vectors $u,v$ of $\Mtr$ is unique and all elements in $v$ are positive (up to sign flipping). (ii) $u=\Mtr v/\sigma_1$ gives that $u$ is non-negative and $u_i>0$ iff there exists $i\le l<i+k$ such that $l\in \Str$. (iii) For $x=\pm e_l$, $v^TM_{x,y}^T u>0$ iff there exists $l-k<i\le l$ such that $u_i>0$, which means that there exists $l-k<l'<l+k$ such that $l\in \Str$. So $x$ will be misclassified only if there is no training sample within its $k$-neighborhood.
\fi
The message delivered by Theorem~\ref{thm:err-decomp} is that the upper bound of the asymptotic generalization error depends on whether $\Mtr^T\Mtr$ is primitive and (if yes) how much of the whole dataset is covered by the $k$-neighborhoods of the points in the training set. Next we will discuss the two quantities in Theorem~\ref{thm:err-decomp} separately.

First consider the second term. Let $\Xtr=\{x_1,x_2,...,x_{n}\}$ be the collection of $x$s in the training set $\Dtr$ with $\ntr=n$ and $\Ltr=\{l_1,l_2,...,l_{n}\}$ be the corresponding non-zero positions of $\Xtr$, which are i.i.d. samples from $\uni\iset{d}$. Therefore
%\begin{align*}
$
 \Prb{\forall l'\in \Str, |l'-l|\ge k} 
 = \Prb{\bigcap_{i=1}^{n} |l_i-l|\ge k }
% & = \Prb{|l'-l|\ge k ; l'\sim \uni\iset{d}}^{n} \\
 = \left(\frac{d-k-\min\{k, l, d-l+1\}+1}{d}\right)^{n}$.
%\end{align*}
The second quantity now can be exactly calculated by averaging over all $l\in \iset{d}$. To get a cleaner form that is independent of $l$, we can either further upper bound it by $\left(\frac{d-k}{d}\right)^n$ or approximate it by $\left(\frac{d-2k+1}{d}\right)^{n}$ if $k\ll d$.

Now come back to the first quantity 
%in \eqref{eq:err-decomp}
, which is the probability that $\Mtr^T\Mtr$ is not primitive. Exactly calculating or even tightly upper bounding this quantity seems hard so we derive a sufficient condition for $\Mtr^T\Mtr$ to be primitive so that the probability of its complement can be used to upper bound the probability that $\Mtr^T\Mtr$ is not primitive:
\begin{lemma}
\label{lemma:primitive}
Let $\tOtr$ be the event that there exists $k\le i \le d$ such that both $i-1, i \in \Str$. If $\tOtr$ happens then $\Mtr^T\Mtr$ is primitive.
\end{lemma}

Lemma~\ref{lemma:primitive} says that $\Mtr^T\Mtr$ is primitive if there exist two training samples with adjacent non-zero positions and the positions should be after $k$ due to shifting/padding issues. Thus we have $\Prb{\Omega^c(\Mtr^T\Mtr)}\le \Prb{\tOtr^c}$. Calculating the quantity $\Prb{\tOtr^c}$ which is the probability that no adjacent non-zero positions after $k$ appear in a randomly sampled training set with size $n$, however, is still hard so we empirically estimate $\Prb{\tOtr^c}$. Figure~\ref{fig:ana1} shows that  $\Prb{\tOtr^c}$ has a lower order than the quantity $\left(\frac{d-2k+1}{d}\right)^{n}$ as $n$ goes larger so the second term in Theorem~\ref{thm:err-decomp} becomes dominating in the generalization bound. We give a rough intuition about this: let $s_n$ be the expected number of unique samples when we uniformly draw $n$ samples from $1,...,d$, then $s_n\rar d$ as $n\rar \infty$. The available slots for the $n+1$-th sample not creating adjacent pairs is at most $d-s_n$. So the total probability of not having adjacent pairs can be roughly upper bounded by $\prod_{i=1}^n \frac{d-s_n}{d}$. Taking the ratio to $\left(\frac{d-2k+1}{d}\right)^{n}$ gives $\prod_{i=1}^n \frac{d-s_n}{d-2k+1}$, which goes to zero as $n\rar \infty$ and $s_n\rar d$.

%\begin{wrapfigure}{R}{0.50\textwidth}
%\vspace{-0.0cm}
%\begin{minipage}{0.49\textwidth}
\begin{figure}[h]
\centering
\subfigure[Ratio Err\_1/Err\_2]{
\includegraphics[width=0.45\columnwidth]{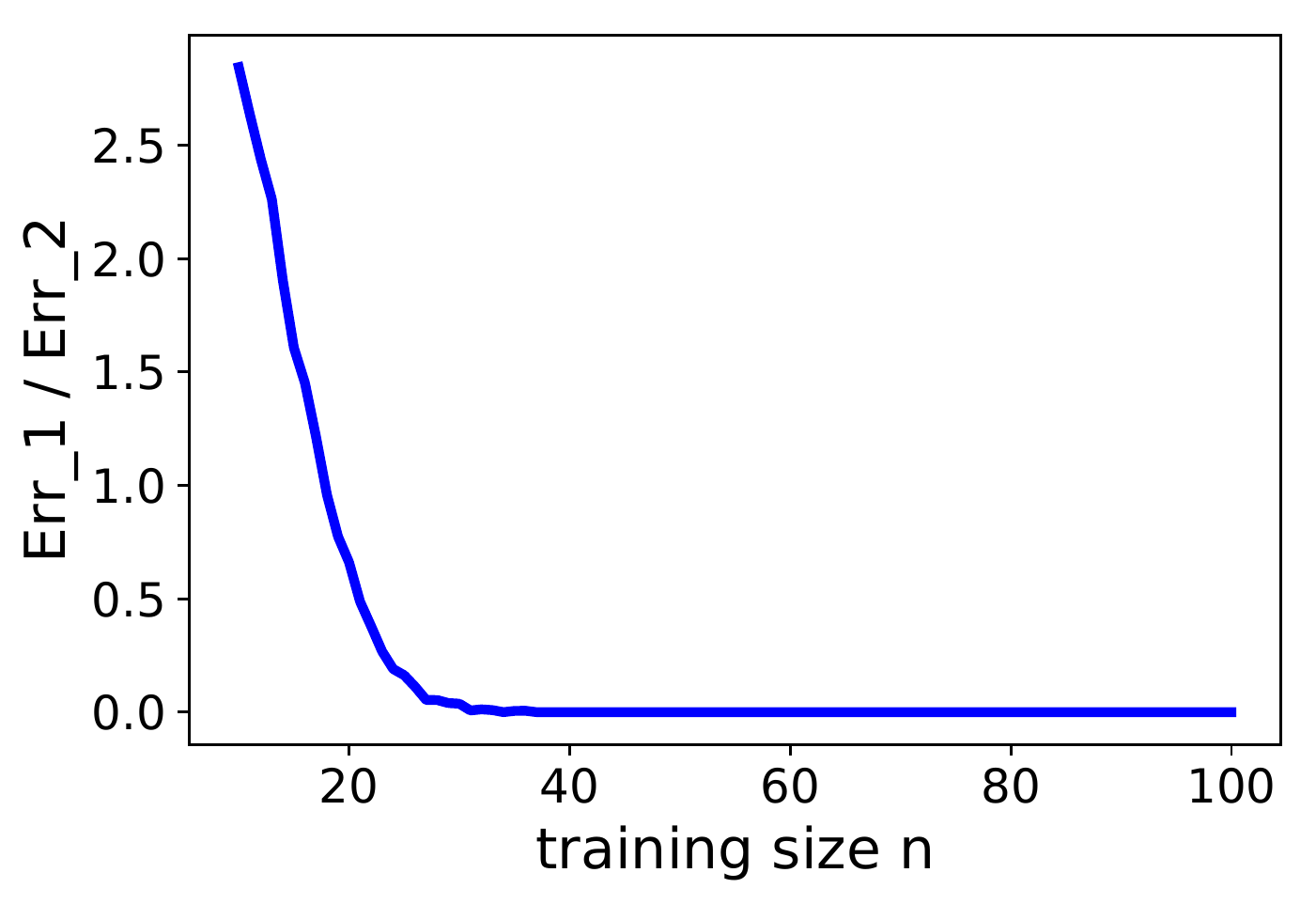}
\label{fig:ana1}
}
\subfigure[Conv-$5$ v.s. $1$-Layer]{
\includegraphics[width=0.45\columnwidth]{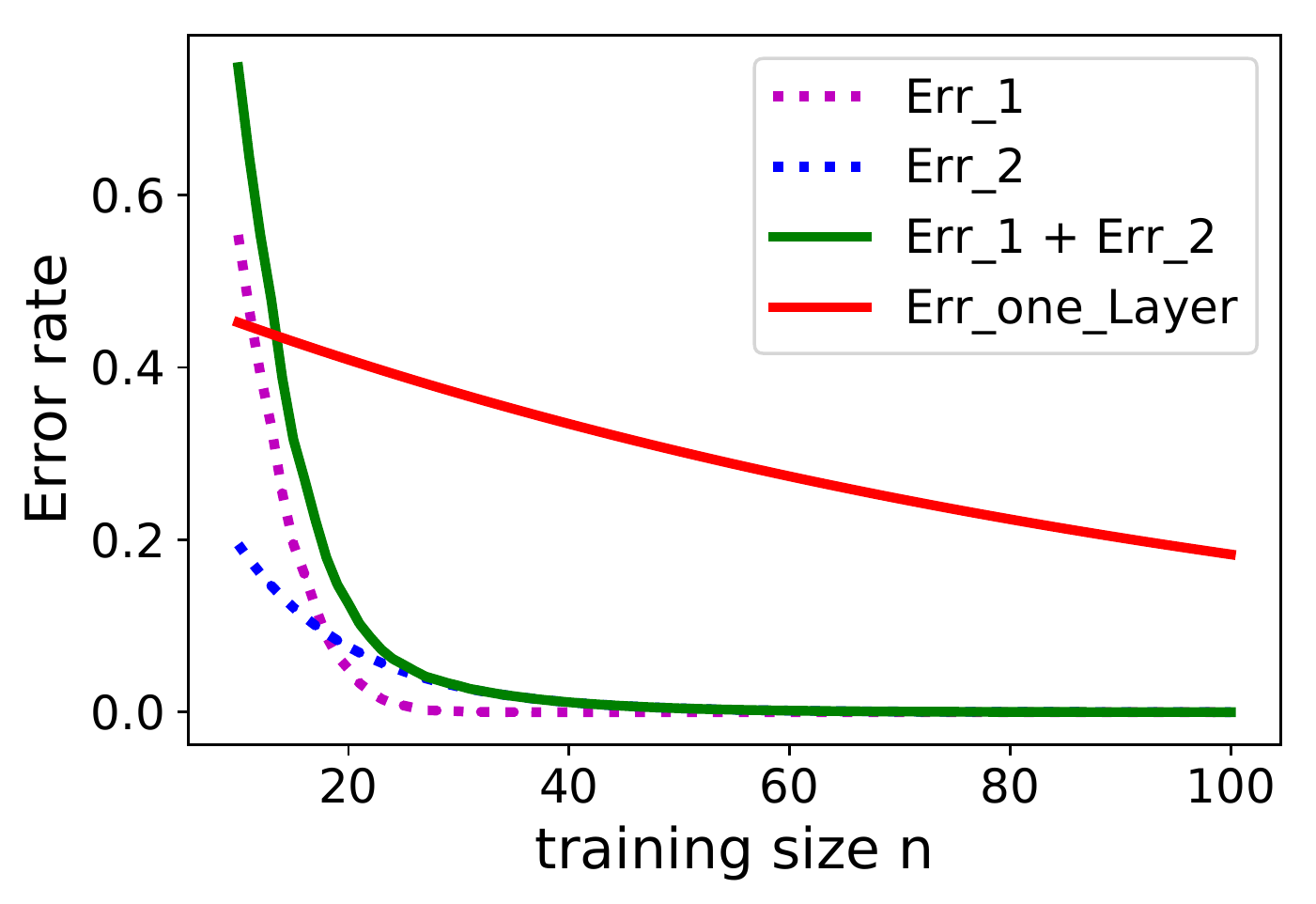}
\label{fig:ana2}
}
\vspace{-0.2cm}
\caption{ Visualizing the calculated/estimated generalization errors. Err\_1 denotes the estimate of $\Prb{\tOtr^c}$, Err\_2 denotes $\frac{1}{2}\left(\frac{d-2k+1}{d}\right)^{n}$ and Err\_one\_Layer denotes $\frac{1}{2}\left(\frac{d-1}{d}\right)^{n}$, where we set $d=100$, $k=5$, and each estimate for $\Prb{\tOtr^c}$ is based on repeatedly sampling $n$ points from $\uni\iset{d}$ for 10,000 times. 
}
\vspace{-0.5cm}
\label{fig:ana}
\end{figure}
%\end{minipage}
%\end{wrapfigure}

\subsubsection{Comparing with Model-$1$-Layer}

We compare the second term of Theorem~\ref{thm:err-decomp}, which is approximately $\frac{1}{2}\left(\frac{d-2k+1}{d}\right)^{n}$, with the generalization error of Model-$1$-Layer. Assume all elements in the single layer weights are initialized independently with some distribution centering around $0$. In each step of gradient descent  $w_i$ is updated only if $x=\pm e_i$ is in the training set. So for any $(x,y)$ in the whole dataset, it is guaranteed to be correctly classified only if  $\pm e_i$ appears in the training set, otherwise it has only a half chance to be correctly classified due to random initialization. Then the generalization error can be written as
%\begin{align*}
$\Err_{\mathrm{1Layer}} = \frac{1}{2}\EE{l\sim \uni\iset{d} }{\Prb{\forall l'\in \Str, l'\ne l}}  = \frac{1}{2}\left(\frac{d-1}{d}\right)^n $
%\,.
%\end{align*} 
The two error rates are the same when $k=1$, which is expected, and $\frac{1}{2}\left(\frac{d-2k+1}{d}\right)^{n}$ is smaller when $k>1$. 

To see how much we save on the sample complexity by using Model-Conv-$k$ to achieve a certain error rate $\epsilon$ we let $\frac{1}{2}\left(\frac{d-2k+1}{d}\right)^{n}=\epsilon$, which gives
$
n =  \frac{1}{\log d - \log (d-2k+1)}\log\frac{1}{2\epsilon} 
$
and $
\lim_{d\rar\infty} \frac{n}{d} 
 = \frac{1}{2k-1}\log\frac{1}{2\epsilon}$.  So the sample complexity for using Model-Conv-$k$ is approximately $\frac{d}{2k-1}\log\frac{1}{2\epsilon}$ when $k\ll d$ while we need $d\log\frac{1}{2\epsilon}$ samples for Model-$1$-Layer. Model-Conv-$k$ requires approximately $2k-1$ times fewer samples when $k\ll d$ and $\epsilon$ is small enough such that the first part in Theorem~\ref{thm:err-decomp} is negligible.  

Now take the first term in Theorem~\ref{thm:err-decomp} into consideration by adding up the empirically estimated $\Prb{\tOtr^c}$ and $\frac{1}{2}\left(\frac{d-2k+1}{d}\right)^{n}$ as an upper bound for $\Err^\infty_{\mathrm{Convk}}$ then compare the sum with $\Err_{\mathrm{1Layer}} $. Figure~\ref{fig:ana2} shows that the estimated upper bound for $\Err^\infty_{\mathrm{Conv5}}$ is clearly smaller than $\Err_{\mathrm{1Layer}} $ when $n$ is not too small. This difference is well aligned with our empirical observation in Figure~\ref{fig:gcomp1} where the two models perform similarly when $n$ is small and Model-Conv-$k$ outperforms Model-$1$-Layer when $n$ grows larger. 

Theorem~\ref{thm:err-decomp} and Lemma~\ref{lemma:primitive} show that when there exist $l, l'\in \Str$ such that $|l-l'|=1$ then the training samples in $\Dtr$ generalize to their $k$-neighbors. We argue that this generalization bias itself requires some samples to be built up, which means that achieving $k$-neighbors generalization requires some condition hold for $\Str$. Having $l, l'\in \Str$ such that $|l-l'|=1$ is a sufficient condition but not a requirement. Now we derive a necessary condition for this generalization advantage:

\begin{proposition}
If for all $l\in \Str$ we have $l\ge k$ and for any $l, l'\in \Str$ we have $|l-l'|\ge 2k$ then this $k$-neighbor generalization does not hold for Conv-$k$ in Task-Cls. Actually, under this condition and $w_1^0 \sim \normal(0, b^2I)$, there is no generalization advantage for Model-Conv-$k$ compared to Model-$1$-Layer.
\label{prop:nece-cond}
\end{proposition}

Proposition~\ref{prop:nece-cond} states that when the training samples are too sparse Model-Conv-$k$ provides the same generalization performance as Model-$1$-Layer. Together with Theorem~\ref{thm:err-decomp} and Lemma~\ref{lemma:primitive} our results reveal a very interesting fact that, unlike traditional regularization techniques, the generalization bias here requires a certain amount of training samples before saving the sample complexity effectively.
\section{Experiments}
\label{sec:exp}

In this section we empirically investigate the relationship between our analysis and the actual performance in experiments (Figure~\ref{fig:gcomp}). Recall that we made two major surrogates during our analysis: (i) We consider the extreme hinge loss $\loss(w;x,y)=-yf_w(x)$ instead of the typically used $\loss(w;x,y)=(1-yf_w(x))_+$. (ii) We consider the asymptotic weights $w^\infty$ instead of $w^t$. Now we study the difference caused by these surrogates. We compare the following three quantities: (a) the empirical estimate for the asymptotic error $\Err^\infty_{\mathrm{Convk}}$ using Theorem~\ref{thm:asym-err-bound} by computing SVD of sampled $\Mtr$s, (b) the test errors by training with the extreme hinge loss and (c) the real hinge loss. \footnote{We also tried cross entropy loss with sigmoid and found no much difference from using the hinge loss.} The results are shown in Figure~\ref{fig:acomp}. 

%\vspace{-0.1cm}
%\begin{wrapfigure}{R}{0.6\textwidth}
%\vspace{-1.9cm}
%\begin{minipage}{0.59\textwidth}
\begin{figure}[h]
\vskip -0.1in
\centering
\subfigure[Task-Cls]{
\includegraphics[width=0.45\columnwidth]{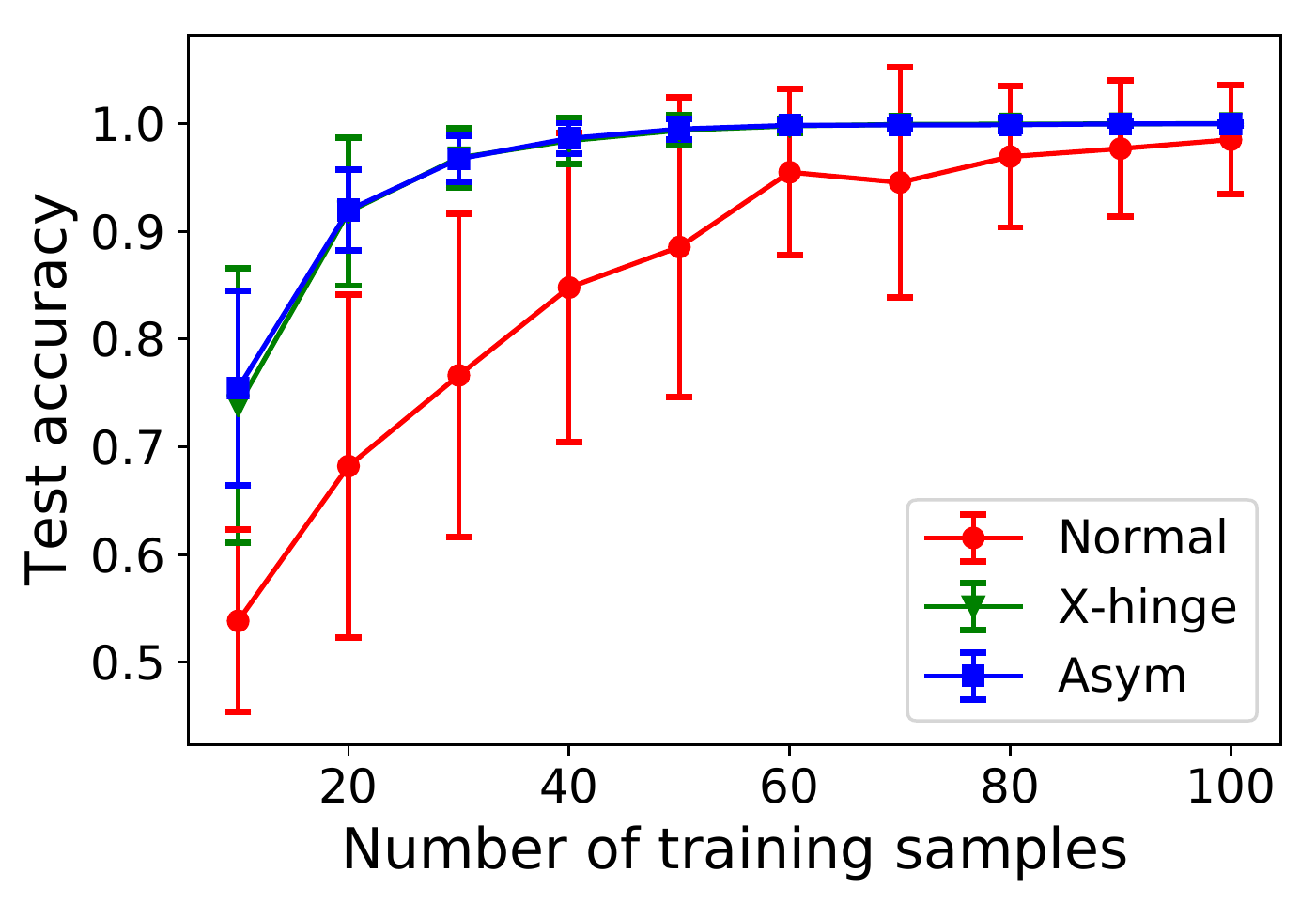}
\label{fig:acomp1}
}
\subfigure[Task-1stCtrl]{
\includegraphics[width=0.45\columnwidth]{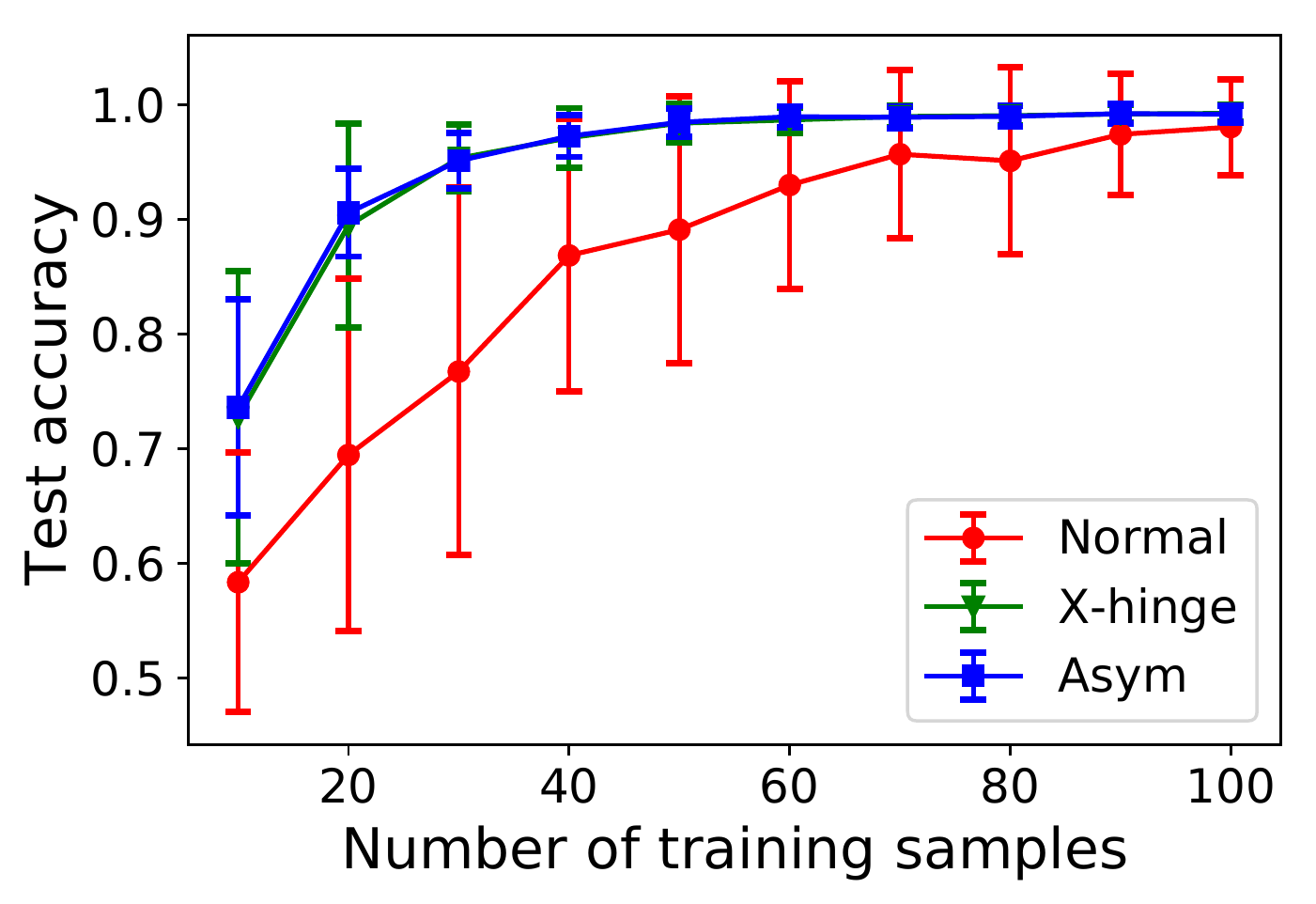}
\label{fig:acomp2}
}
\subfigure[Task-3rdCtrl]{
\includegraphics[width=0.45\columnwidth]{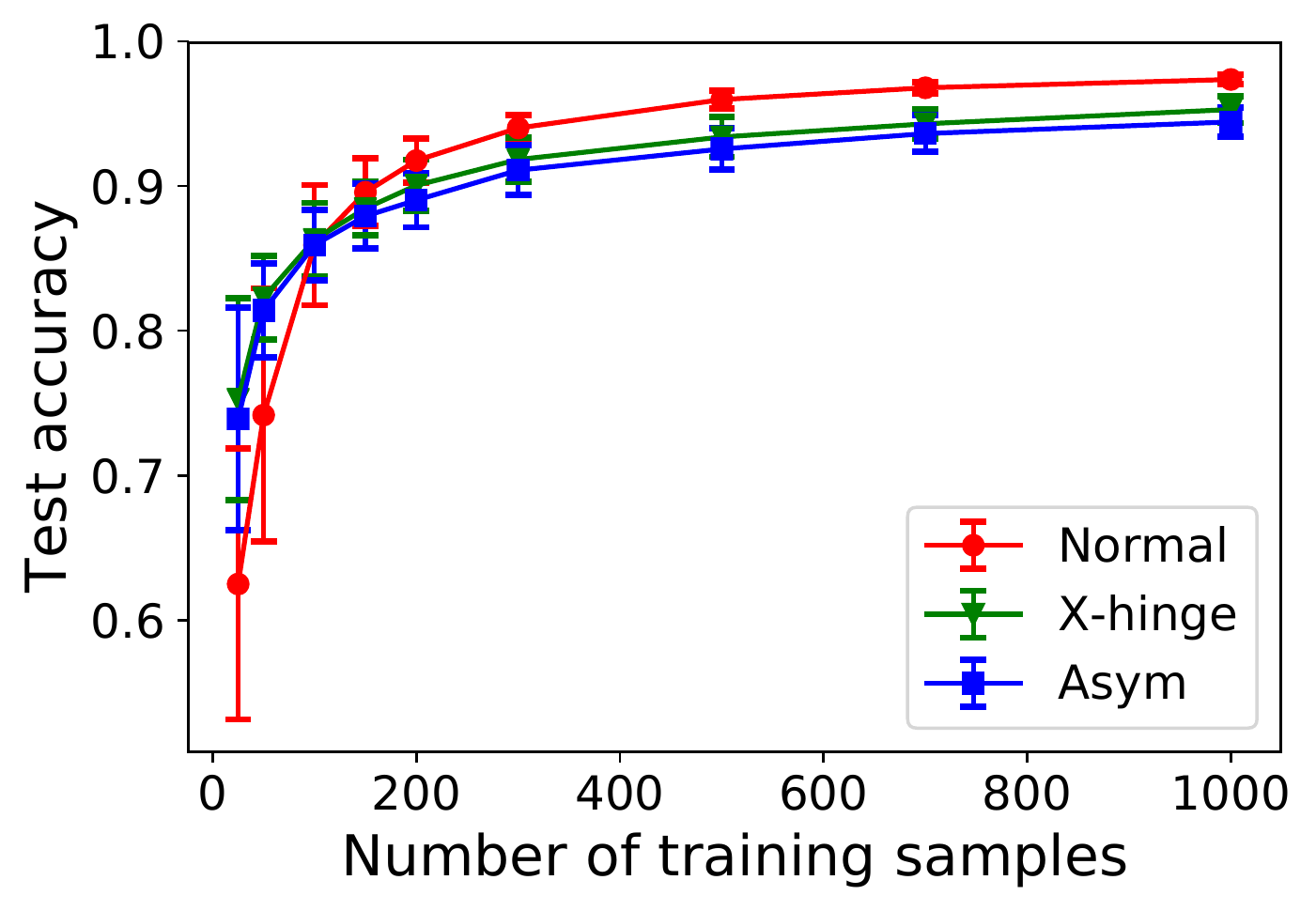}
\label{fig:acomp3}
}
\caption{
Comparing estimated asymptotic error (Asym) v.s. finite time extreme hinge loss (X-hinge) v.s. normal hinge loss (Normal) with different sizes of training data. For using normal hinge loss training stops when training loss goes to $0$ while for extreme hinge loss we train the model for 1000 steps. The other settings remain the same as the experiments shown in Figure~\ref{fig:gcomp}. 
}
\label{fig:acomp}
\vskip -0.1in
\end{figure}
%\end{minipage}
%\vspace{-0.2cm}
%\end{wrapfigure}

It can be seen from Figure~\ref{fig:acomp} that there is not much difference between the estimated quantity in Theorem~\ref{thm:asym-err-bound} by SVD and the actual test error by training with the extreme hinge loss $\loss(w;x,y)=-yf_w(x)$, which verifies our derivation in Section~\ref{sec:ana}. It is also shown that, especially in Task-Cls and Task-1stCtrl, the asymptotic quantity can be viewed as an upper confidence bound for the actual performance with the normal hinge loss. The asymptotic quantity has a much lower variance which only comes from the randomization of the training set so the high variance with the normal hinge loss is caused by random initialization and good initializations would perform closer to the asymptotic quantity than the bad ones. To verify this, we fix the training set and compare the performance of the two losses at each training step $t$ with difference initial weights $w^0$. Figure~\ref{fig:winit1}
\footnote{Results for the other two tasks are put in the appendix.} shows the convergence of training/test accuracies with difference losses. With the normal hinge loss, the test performance remains the same once the training loss reaches $0$. With the extreme hinge loss, the test performance is still changing even after the training data is fitted and eventually converges to $\Err^\infty_{\mathrm{Convk}}$. As we can see, there is a difference in how fast the direction of weight $w^t$ converges (in terms of test accuracy) to its limit $w^\infty$ defined in Lemma~\ref{lemma:w-asym} with different initialization when using the extreme hinge loss. We further argue that this variance is strongly correlated with the variance in the generalization performance under the normal hinge loss, as shown in Figure~\ref{fig:winit2}, from which we can see how well a model trained using the normal hinge loss with some $w^0$ generalizes depends on how fast $w^t$ converges to its limit direction using the extreme hinge loss. %This correlation means that we may be able to tell how good a random initialization is in practice by analyzing its finite time convergence under the extreme hinge loss. 

%\begin{wrapfigure}{R}{0.55\textwidth}
%\vspace{-0.4cm}
%\begin{minipage}{0.54\textwidth}
\begin{figure}[h]
\vskip -0.1in
\centering
\subfigure[Convergence with $t$.]{
\includegraphics[width=0.45\columnwidth]{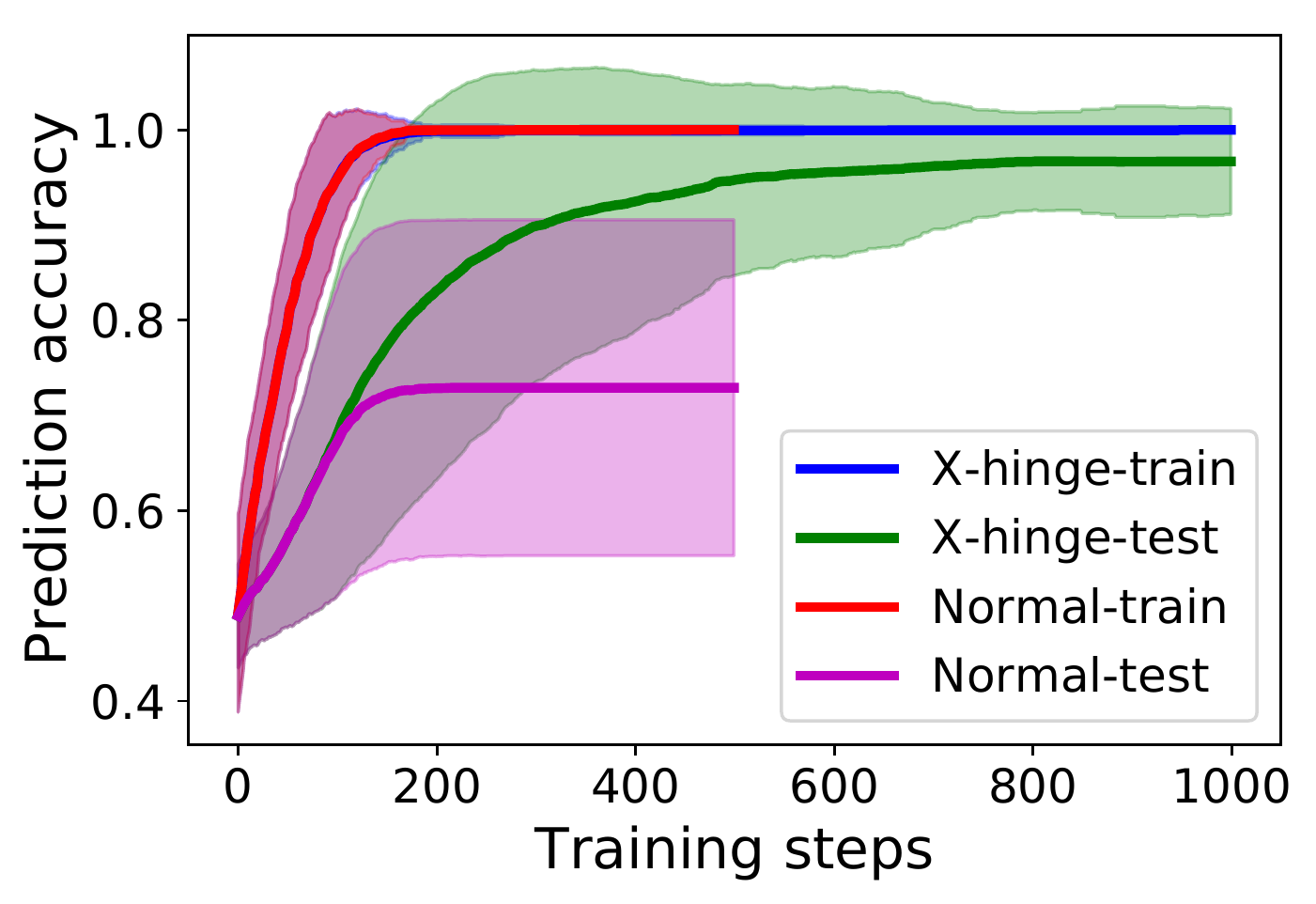}
\label{fig:winit1}
}
\subfigure[Correlation at $t=150$.]{
\includegraphics[width=0.45\columnwidth]{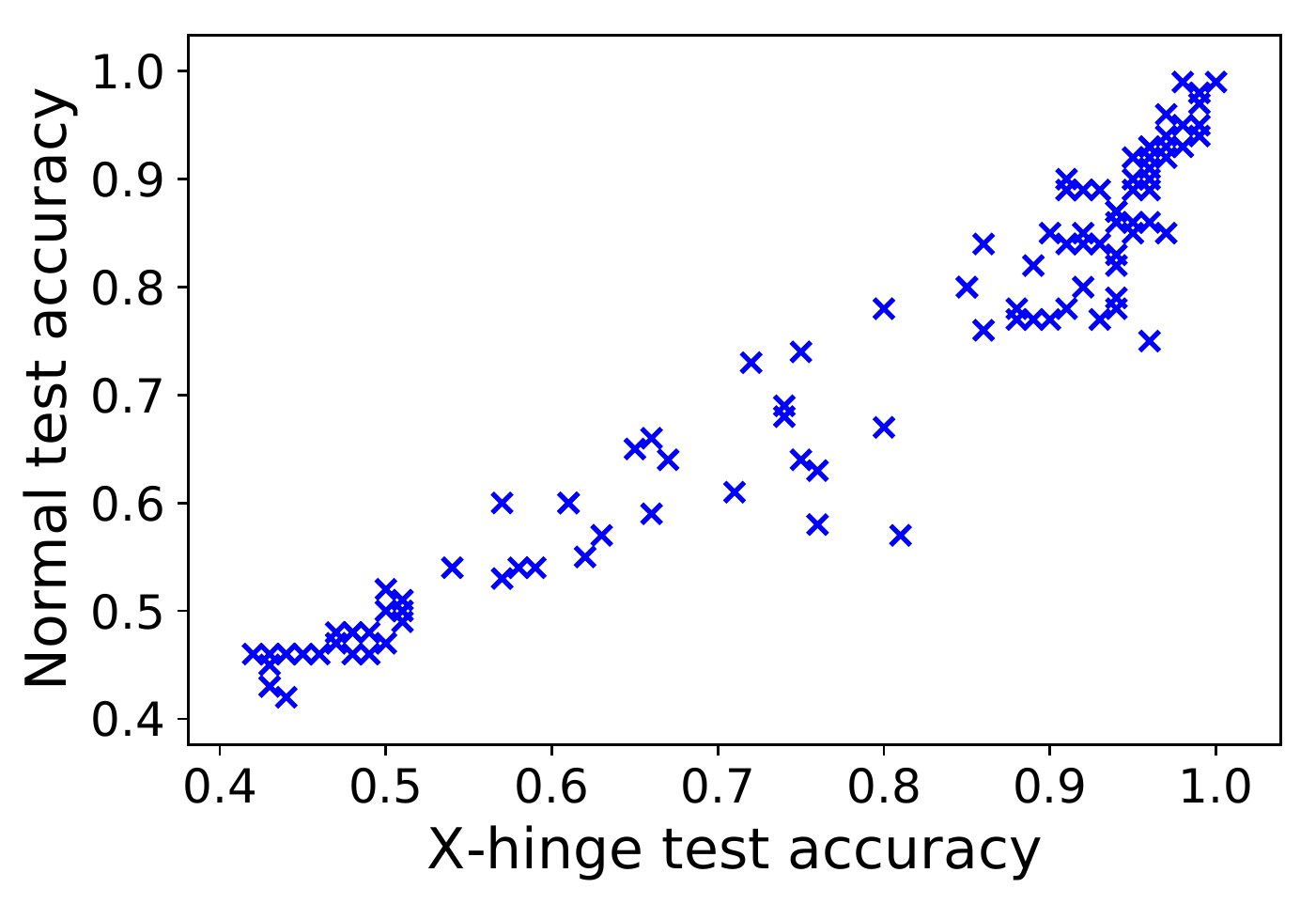}
\label{fig:winit2}
}
\caption{
The effect of weight initialization in Task-Cls. We fix $d=100$, $n=30$ and train Model-Conv-$5$ with $100$ different random initializations using both losses. $w^0$ is uniformly sampled from $[-b, b]^{d+k}$. 
}
\label{fig:winit}
\vskip -0.1in
\end{figure}
%\end{minipage}
%\vspace{-0.5cm}
%\end{wrapfigure}

One may wonder that whether X-hinge always generalizes better than the normal hinge under gradient descent as in Task-Cls. However, this is not true in Task-3rdCtrl, where the limit direction is better when $\ntr$ is small but worse when $\ntr$ is large, according to Figure~\ref{fig:acomp3}. The reason is that the limit direction $w^\infty$ may not be able to separate the training set\footnote{See appendix for an example.}. This indicates that the potential generalization ``benefit" from the convolution layer may actually be a bias.  
%Investigating how the bias depends on the data distribution would be an interesting future direction.
\vspace{-0.3cm}
\section{Related Work}
\label{sec:rel}

Among all recent attempts that try to explain the behavior of deep networks our work is distinct in the sense that we study the generalization performance that involves the interaction between gradient descent and convolution. For example, \cite{du2017convolutional} study how gradient descent learns convolutional filters but they focus on optimization instead of generalization. 
Several recent works study the generalization bias of gradient descent \citep{hardt2015train, dziugaite2017computing, brutzkus2017sgd, soudry2017implicit} but they are not able to explain the advantage of convolution in our examples. \cite{hardt2015train} bounds the stability of stochastic gradient descent within limited number of training steps. 
\cite{dziugaite2017computing} proposes a non-vacuous bound that relies on the stochasticity of the learning process. Neither limited number of training steps or stochasticity is necessary to achieve better generalization in our examples. 
Similarly to our work, \cite{soudry2017implicit} study the convergence of $w/\norm{w}_2$ under gradient descent. However, their work is limited to single layer logistic regression and their result shows that the linear separator converges to the max-margin one, which does not indicate good generalization in our cases. 
\cite{gunasekar2018implicit} also study the limit directions of multi-layer linear convolutional classifiers under gradient descent. Their result is not directly applicable to ours since they consider loopy convolutional filters with full width $k=d$ while we consider filters with $k\ll d$ and padding with $0$. Our setting of filters is closer to what people use in practice. Moreover, \cite{gunasekar2018implicit} does not provide any generalization analysis while we show that the limit direction of the convolutional linear classifier provides significant generalization advantage on some specific tasks. \cite{brutzkus2017sgd} shows that optimizing an over-parametrized 2-layer network with SGD can generalize on linearly separable data. Their work is limited to only training the first fully connected layer while we study jointly training two layers with convolution. 
Another thread of work \citep{bartlett2017spectrally, neyshabur2017pac, neyshabur2017exploring} tries to develop novel complexity measures that are able to characterize the generalization performance in practice. These complexities are based on the margin, norm or the sharpness of the learned model on the training samples.  Taking Task-Cls as an example, the linear classifier with the maximum margin or minimum norm will place $0$ on the weights where there are no training samples, which is undesirable in our case, while the sharpness of the learned model in terms of training loss contains no information about how it behaves on unseen samples. So none of these measures can be applied to our scenario. \citep{fukumizu1999generalization, saxe2013exact, pasa2014pre, advani2017high} study the dynamics of linear network but these results do not apply in our case due to difference loss and network architecture: \citep{fukumizu1999generalization, saxe2013exact, advani2017high} study fully connected networks with L2 regression loss  while \cite{pasa2014pre} considers recurrent networks with reconstruction loss.

%\if0
\section{Conclusion}
\label{sec:conc}   

We analyze the generalization performance of two layer convolutional linear classifiers trained with gradient descent. Our analysis is able to explain why, on some simple but realistic examples, adding a convolution layer can be more favorable than just using a single layer classifier even if the data is linearly separable. Our work can be a starting point for several interesting future direction: 
(i) 
Closing the gaps in normal hinge loss v.s. the extreme one as well as asymptotic analysis v.s. finite time analysis. The latter may be able to characterize how good a weight initialization is.
% are good. Although strongly correlated, there is still a gap between using the normal hinge loss and the extreme one. The next step may be to have a finite time analysis of the extreme hinge loss, which will be closer to the normal hinge loss (Figure~\ref{fig:winit1}) and able to characterize which weight initializations are good or bad. 
%Another question is how optimizing the normal hinge loss can be viewed as a balance between being consistent with the training data and exploiting the bias provided by the gradient dynamics under the extreme hinge loss, as observed in Task-3rdCtrl (Figure~\ref{fig:acomp3}). 
(ii) Another interesting question is how we can interpret the generalization bias as a prior knowledge. We conjecture that the jointly trained filter works as a data adaptive bias as it requires a certain amount of data to provide the generalization bias (supported by Proposition~\ref{prop:nece-cond}. 
(iii) Other interesting directions include studying the choice of $k$, making practical suggestions based on our analysis and bringing in more factors such as feature extraction, non-linearity and pooling.  
%\fi

%\subsubsection*{References}

\newpage
\bibliographystyle{apalike}
\bibliography{aistats_2019}

\newpage
\appendix
%\section{Proof of Lemma~\ref{lemma:w}}

\section{Proof of Lemma~\ref{lemma:w-asym}}

We first introduce the following Lemma, which shows that $w_1^t$ and $w_2^t$ can be written in closed-forms in terms of $(w_1^0, w_2^0, \Mtr, \alpha, t)$:

\begin{lemma}
\label{lemma:w}
Let $\Mtr=U\Sigma V^T$ be (any of) its SVD such that $U\in \real^{d\times k}, \Sigma\in \real^{k\times k}, V\in \real^{k\times k}$, $U^TU=V^TV=VV^T=I$. Then for any $t\ge 0$
\begin{align*}
& w_1^t = \frac{1}{2}V\left(\Lambda^{+,t} V^T w_1^0
+ \Lambda^{-,t} U^T w_2^0 \right) \,,\\
& w_2^t = \frac{1}{2}U\left(\Lambda^{-,t} V^T w_1^0
+ \Lambda^{+,t} U^T w_2^0\right)  - UU^T w_2^0 + w_2^0 \,. \addeq\label{eq:weights}
\end{align*} 
where we define $\Lambda^{+,t} = (I+\alpha\Sigma)^t + (I-\alpha\Sigma)^t$ and 
$\Lambda^{-,t} = (I+\alpha\Sigma)^t - (I-\alpha\Sigma)^t$.

\end{lemma}

\begin{proof}

We start with stating the following facts for $\Lambda^{+,t}$ and $\Lambda^{-,t}$:

$\Lambda^{+,0} = 2I$, $\Lambda^{-,0}=0$ and for any $t\ge 0$
\begin{align*}
& \Lambda^{+, t+1} = \Lambda^{+,t} + \alpha\Sigma\Lambda^{-,t} \,,\\
& \Lambda^{-, t+1} = \Lambda^{-,t} + \alpha\Sigma\Lambda^{+,t} \,. 
\end{align*}
Now we prove \eqref{eq:weights} by induction.
When $t=0$, $w_1^0 = VV^Tw_1^0$ and $w_2^0 = UU^Tw_2^0-UU^Tw_2^0 + w_2^0$ so \eqref{eq:weights} holds for $t=0$. Assume Lemma~\eqref{eq:weights} holds for $t$ then consider the next step $t+1$:
\begin{align*}
w_1^{t+1} & = w_1^t + \alpha \Mtr^T w_2^t \\
& = \frac{1}{2}V\left( \Lambda^{+,t}  V^T w_1^0 
 + \Lambda^{-,t} U^T w_2^0 \right) \\
& + \alpha V\Sigma U^T \left(
\frac{1}{2}U\left(\Lambda^{-,t} V^T w_1^0
+ \Lambda^{+,t} U^T w_2^0\right) \right. \\
& \left. - UU^T w_2^0 + w_2^0 \right) \\
& = \frac{1}{2}V\left(
\Lambda^{+,t}  V^T w_1^0 + \Lambda^{-,t} U^T w_2^0 \right.\\
& \left. + \alpha\Sigma \Lambda^{-,t} V^T w_1^0 + \alpha\Sigma \Lambda^{+,t} U^T w_2^0 \right) \\
& = \frac{1}{2}V\left(\Lambda^{+,t+1} V^T w_1^0
+ \Lambda^{-,t+1} U^T w_2^0 \right) \,.
\end{align*}
Similarly, we can show
\begin{align*}
w_2^{t+1}
& = w_2^t + \alpha \Mtr w_1^t\\
& = \frac{1}{2}U\left(\Lambda^{-,t+1} V^T w_1^0
+ \Lambda^{+,t+1} U^T w_2^0\right)\\
&  - UU^T w_2^0 + w_2^0 \,.
\end{align*}
Thus \eqref{eq:weights} holds for all $t\ge 0$.
\end{proof}

\begin{proof}[Proof of Lemma~\ref{lemma:w-asym}]
Taking $w_2^0=0$ in Lemma~\ref{lemma:w} we can write $w_1^t = \frac{1}{2}V\Lambda^{+,t}  V^T w_1^0$ and $w_2^t = \frac{1}{2}U \Lambda^{-,t}  V^T w_1^0 $

For $1\le i \le m$, $\sigma_i=\sigma_1$ thus
\begin{align*}
\lim_{t\rar+\infty} \frac{(1+\alpha\sigma_i)^t}{(1+\alpha\sigma_1)^t} = 1 \,. \addeq\label{eq:lim-1}
\end{align*} 
For $m<i\le k$, $\sigma_i<\sigma_1$ thus
\begin{align*}
\lim_{t\rar+\infty} \frac{(1+\alpha\sigma_i)^t}{(1+\alpha\sigma_1)^t} = 0 \,.\addeq\label{eq:lim-2}
\end{align*}
For any $1\le i \le k$, we have
$\frac{1-\alpha\sigma_i}{1+\alpha\sigma_1} \le \frac{1}{1+\alpha\sigma_1} < 1$
and
$\frac{1-\alpha\sigma_i}{1+\alpha\sigma_1} \ge \frac{1-\alpha\sigma_1}{1+\alpha\sigma_1} = -1 + \frac{2}{1+\alpha\sigma_1} > -1$
thus
\begin{align*}
\lim_{t\rar+\infty} \frac{(1-\alpha\sigma_i)^t}{(1+\alpha\sigma_1)^t} = 0\,.\addeq\label{eq:lim-3}
\end{align*}
Applying \eqref{eq:lim-1}---\eqref{eq:lim-3} to compute the limits in \eqref{eq:weights-limit} gives the result in Lemma~\ref{lemma:w-asym}.
\end{proof}

\section{Proof of Theorem~\ref{thm:asym-err-bound}}

\begin{proof}
For any vector $z\in \real^m$ such that $\norm{z}_2=1$, we have 
\begin{align*}
\Mtr V_{:m}z = U\Sigma V^TV_{:m}z = \sigma_1 U_{:m}z \,, \\
\Mtr^T U_{:m}z = V\Sigma U^T U_{:m}z = \sigma_1 V_{:m}z \,,
\end{align*}
Since
\begin{align*}
\norm{V_{:m}z}_2^2=z^TV_{:m}^TV_{:m}z=1\,,\\
\norm{U_{:m}z}_2^2=z^TU_{:m}^TU_{:m}z=1
\end{align*} 
we know that $(U_{:m}z, V_{:m}z)$ is also a pair of left-right singular vectors with singular value $\sigma_1$. Therefore, when $V_{:m}^Tw_1^0\in\real^m$ is non-zero $\left(\frac{U_{:m} V_{:m}^Tw_1^0}{\norm{V_{:m}^Tw_1^0}_2}, \frac{V_{:m} V_{:m}^Tw_1^0}{\norm{V_{:m}^Tw_1^0}_2}\right)$ is also such a pair.  Following \eqref{eq:asym-err-1} we have
\begin{align*}
\frac{\F^\infty(x,y,w_1^0,\Dtr)}{\norm{V_{:m}^Tw_1^0}_2^2} 
& = \left( \frac{V_{:m} V_{:m}^Tw_1^0}{\norm{V_{:m}^Tw_1^0}_2}\right)^T M_{x,y}^T \left( \frac{U_{:m} V_{:m}^Tw_1^0}{\norm{V_{:m}^Tw_1^0}_2}\right)  \\
& \ge \min_{(u,v)\in UV_1^{\Mtr}} v^T M_{x,y}^T u \addeq\label{eq:thm-asym-err-pf-1}
\end{align*}
for any $w_1^0$ such that $V_{:m}^Tw_1^0\ne \vec{0}$.

When $w_1^0\sim\normal(0, b^2I_k)$, for any fixed $V_{:m}$ satisfying $V_{:m}^TV_{:m}=I_m$, the random variable $V_{:m}^Tw_1^0$ also follows a normal distribution:
\begin{align*}
\E{V_{:m}^Tw_1^0 (V_{:m}^Tw_1^0)^T} = V_{:m}^T\E{w_1^0{w_1^0}^T} V_{:m} =b^2I_m
\end{align*}
hence  $V_{:m}^Tw_1^0 \sim\normal(0, b^2I_m)$.

Applying the fact that $\bErr(\cdot)\le 1$ is non-increasing and $\bErr(\alpha x) = \bErr(x)$ for any $\alpha>0$ we can upper bound \eqref{eq:asym-err-1} by
\begin{align*}
& \Err^\infty_{\mathrm{Convk}}(\D) \\
& = \EE{w_1^0, \Dtr, (x,y)}{\bErr\left(\F^\infty(x,y,w_1^0,\Dtr)\right)} \\
& = \EE{\Dtr, (x,y)}{\EE{w_1^0}{\bErr\left(\F^\infty(x,y,w_1^0,\Dtr)\right)}} \\
& = \EE{\Dtr, (x,y)}{ \Prb{V_{:m}^Tw_1^0 = \vec{0}}\EE{w_1^0}{\bErr\left(\F^\infty\right)\big|V_{:m}^Tw_1^0 = \vec{0}} 
\right. \\
& \left. + \Prb{V_{:m}^Tw_1^0 \ne \vec{0}}\EE{w_1^0}{\bErr\left(\F^\infty\right)\big|V_{:m}^Tw_1^0 \ne \vec{0}}} \\
& = \EE{\Dtr, (x,y)}{\EE{w_1^0}{\bErr\left(\frac{\F^\infty}{\norm{V_{:m}^Tw_1^0}_2^2}\right)\big|V_{:m}^Tw_1^0 \ne \vec{0}}} \\
& \le \EE{\Dtr, (x,y)}{\bErr\left(\min_{(u,v)\in UV_1^{\Mtr}} v^T M_{x,y}^T u \right) } \,.
\end{align*}%\todoy{Briefly explain this quantity}
\end{proof}

\section{Perron-Frobenius Theorem}

Let $A\in \real^{k\times k}$ be a non-negative square matrix\footnote{\url{https://en.wikipedia.org/wiki/Perron-Frobenius\_theorem} .}:\vspace{-.1in}
\begin{itemize}\addtolength\itemsep{-.5em}
\item{
\textbf{Definition:} $A$ is \emph{primitive} if there exists a positive integer $t$ such that $A^t_{ij}>0$ for all $i,j$. 
}
\item{
\textbf{Definition:} $A$ is \emph{irreducible} if for any $i,j$ there exists a positive integer $t$ such that $A^t_{ij}>0$.
}
\item{
\textbf{Definition:} Its \emph{associated graph} $\G_A=(V,E)$ is defined to be a directed graph with $V=\{1,...,k\}$ and $(i,j)\in E$ iff $A_{ij}\ne 0$. $\G_A$
 is said to be \emph{strongly connected} if for any $i,j$ there is path from $i$ to $j$.
}
\item{
\textbf{Property:} $A$ is irreducible iff $\G_A$
 is strongly connected.
}
\item{
\textbf{Property:} If $A$ is irreducible and has at least one non-zero diagonal element then $A$ is primitive.
}
\item{
\textbf{Property:} If $A$ is primitive then its first eigenvalue is unique ($\lambda_1>\lambda_2$) and the corresponding eigenvector is all-positive (or all-negative up to sign flipping).
}
\end{itemize}

\section{Proof of Theorem~\ref{thm:err-decomp}}

\begin{proof}
Following \eqref{eq:asym-err-1} and let 
\[
\hErr(x,y,\Dtr)= \bErr\left(\min_{(u,v)\in UV_1^{\Mtr}} v^T M_{x,y}^T u \right) \le 1
\] we have
\begin{align*}
& \Err^\infty_{\mathrm{Convk}}(\D) 
 \le \EE{\Dtr, (x,y)}{\hErr(x,y,\Dtr)} \\
& = \EE{\Dtr, (x,y)}{\left(\I{\Omega^c(\Mtr^T\Mtr)} + \I{\Omega(\Mtr^T\Mtr)}\right) \right. \\
& \left. \hErr(x,y,\Dtr)} \\
%& = \EE{\Dtr, (x,y)}{\I{\Omega^c(\Mtr^T\Mtr)}\hErr(x,y,\Dtr)} \\
%& + \EE{\Dtr, (x,y)}{\I{\Omega(\Mtr^T\Mtr)}\hErr(x,y,\Dtr)} \\
& \le \EE{\Dtr}{\I{\Omega^c(\Mtr^T\Mtr)}} \\
& + \EE{\Dtr, (x,y)}{\I{\Omega(\Mtr^T\Mtr)}\hErr(x,y,\Dtr)} \\
& = \Prb{\Omega^c(\Mtr^T\Mtr)} \\
& + \EE{\Dtr, l\sim \uni\iset{d}}{\I{\Omega(\Mtr^T\Mtr)}\hErr(e_l,1,\Dtr)} \addeq\label{eq:pf-decomp1}
\end{align*}

Now look at the second term in \eqref{eq:pf-decomp1}. If $\Mtr^T\Mtr$ is primitive then its first eigenvalue 
$\lambda_1=\sigma_1^2$ is unique ($\sigma_1>\sigma_2$) and the corresponding eigenvector $v$ is all positive (or all negative if we flip the sign of $v$ and $u$, which does not change the sign of $v^TM^Tu$ thus it is safe to assume $v>0$). $u=\Mtr v/\sigma_1$ gives that $u$ is also unique and non-negative. Since $M_{x,y}$ is also non-negative we have $v^T M_{x,y}^T u \ge 0$ for any $x,y$. Therefore,
\begin{align*}
& \hErr(x,y,\Dtr) = \bErr\left( v^T M_{x,y}^T u \right)\\
& = \I{v^T M_{x,y}^T u <0} + \frac{1}{2}\I{v^T M_{x,y}^T u = 0} \\
& = \frac{1}{2}\I{v^T M_{x,y}^T u = 0} \,.
\end{align*}
From $u=\Mtr v/\sigma_1$ and $v>0$ we know that $u_i>0$ iff there exists $1\le j\le k$ such that $(\Mtr)_{i,j}>0$, which is equivalent to that there exists $i\le l<i+k$ such that $l\in \Str$. Also for $x=e_l$ ($y=1$), according to the definition of $M_{x,y}$ and the fact that $v>0$ we have $v^TM_{e_l,1}^T u>0$ iff there exists $l-k<i\le l$ such that $u_i>0$. So we have 
\begin{align*}
v^TM_{e_l,1}^T u> 0 \iff \exists l'\in \bigcup_{l-k<i\le l} [i,i+k) \text{ s.t. } l'\in \Str 
\end{align*}
Since $v^TM_{e_l,1}^T u\ge 0$ and $\bigcup_{l-k<i\le l} [i,i+k) = (l-k, l+k)$ we have
\begin{align*}
v^TM_{e_l,1}^T u = 0 \iff \forall l'\in \Str, |l'-l|\ge k \,. 
\end{align*} 

Now we have proved that, if $\Mtr^T\Mtr$ is primitive then
\begin{align*}
\hErr(e_l,1,\Dtr) = \frac{1}{2}\I{\forall l'\in \Str, |l'-l|\ge k} \,,
\end{align*}
which means that 
\begin{align*}
\I{\Omega(\Mtr^T\Mtr)}\hErr(e_l,1,\Dtr) \le \frac{1}{2}\I{\forall l'\in \Str, |l'-l|\ge k}
\end{align*}
holds for any $\Dtr$. Therefore
\begin{align*}
& \EE{\Dtr, l\sim \uni\iset{d}}{\I{\Omega(\Mtr^T\Mtr)}\hErr(e_l,1,\Dtr)} \\
& \le \frac{1}{2} \EE{\Dtr, l\sim \uni\iset{d}}{\I{\forall l'\in \Str, |l'-l|\ge k}} \\
& = \frac{1}{2} \EE{l\sim \uni\iset{d}}{\Prb{\forall l'\in \Str, |l'-l|\ge k}} 
\end{align*}
which concludes the proof.
\end{proof}

\section{Proof of Lemma~\ref{lemma:primitive}}

\begin{proof}
If $k\le i \le d$ and $i-1, i \in \Str$ then for any $1\le j\le k$ we have $(\Mtr)_{i-j,j}>0$ and $(\Mtr)_{i-j+1,j}>0$, which also means that for any $1\le j<k$ we have $(\Mtr)_{i-j,j}>0$ and $(\Mtr)_{i-j,j+1}>0$. Since every two adjacent columns have at least one common non-zero position what we have is $(\Mtr^T\Mtr)_{j,j+1}>0$ and $(\Mtr^T\Mtr)_{j+1,j}>0$ for all $1\le j < k$. So its associated graph $\G_{\Mtr^T\Mtr}$ is strongly connected thus $\Mtr^T\Mtr$ is irreducible. It is also true that all diagonal elements of $\Mtr^T\Mtr$ are positive since every column of $\Mtr$ must contain at least one non-zero element. Now we have proved that $\Mtr^T\Mtr$ is primitive because it is irreducible and has at least one non-zero element on its diagonal.    
\end{proof} 

\section{Proof of Proposition~\ref{prop:nece-cond}}

\begin{proof}
Let $n=|\Str|$. Then given the conditions in this proposition we can see that any column in $\Mtr$ has exactly $n$  non-zero entries with value $1/n$ and any two columns in $\Mtr$ has no overlapping non-zero positions. Hence we have $\Mtr^T\Mtr = \frac{1}{n} I_k$ so that $m=k$ in Lemma~\ref{lemma:w-asym} and $VV^T = I$. Applying  Lemma~\ref{lemma:w-asym} we have $w_1^\infty = w_1^0$ and $w_2^\infty = n \Mtr w_1^0$. Then for any $x = e_l$ we have 
\[
y f_{w^\infty}(x) = {w_1^\infty}^T M_{x,y}^T w_2^\infty = n {w_1^0}^T A_{x}^T\Mtr w_1^0 \,.
\]
For $x$ to be correctly classified we need $y f_{w^\infty}(x) >0$. We will show that this is guaranteed only when $l\in \Str$, i.e. $x$ or $-x \in \Dtr$.

Since for any $l, l'\in \Str$, $|l-l'|\ge 2k$ we know that there exist at most one $l'\in \Str$ such that $|l-l'|<k$. 

If there does not exist such $l'$ then $A_{x}^T\Mtr=0$ and $y f_{w^\infty}(x) =0$, which means $x$ is classified randomly.

If there exists a unique $l'$ such that $|l-l'|<k$ and let $s = |l-l'|$, we have that
\[ 
y f_{w^\infty}(x) = n {w_1^0}^T A_{x}^T\Mtr w_1^0 = \sum_{i=1}^{k-s} w_{1,i}^0 w_{1, i+s}^0 \,.
\]

When $l\in \Str$, which means $s = 0$, we have $ y f_{w^\infty}(x) = {w_1^0}^T w_1^0 > 0$ when $w_1^0\ne 0$ (which holds almost surely).

When $0<s<k$ it is not guaranteed that $\sum_{i=1}^{k-s} w_{1,i}^0 w_{1, i+s}^0 > 0$ under $w_1^0 \sim \normal(0, b^2I)$. Actually we can show that the distribution of this quantity is symmetric around $0$: For any $s$ we can draw a graph with $k$ nodes and every $(i, i+s)$ forms an edge. This graph contains $s$ independent chains so we can choose a set of nodes $S\subset \iset{k}$   such that for any edge exactly one of the two nodes is contained in $S$. Now for any $w_1^0$ if we flip the sign at the positions that belong to $S$ then the sign of $\sum_{i=1}^{k-s} w_{1,i}^0 w_{1, i+s}^0$ is also flipped. With $w_1^0 \sim \normal(0, b^2I)$ this indicates that  $P(\sum_{i=1}^{k-s} w_{1,i}^0 w_{1, i+s}^0 > 0)=1/2$.

Now we have shown that, under the condition in this proposition, a data sample is correctly classified by Conv-$k$ with $w^\infty$ if and only if this sample appears in the training set. Otherwise it has only a half change to be correctly classified. This generalization behavior is exactly the same as Model-$1$-Layer in Task-Cls, which concludes the proof.

\end{proof}

\section{A Supporting Evidence for Interpreting Conv-Filters as a Data Adaptive Bias}

We have shown that, different from typical regularizations, the bias itself may require some samples to be built up (see Figure~\ref{fig:ana2}). We conjecture that convolution layer adds a data adaptive bias: The set of possible filters forms a set of biases. With a few number of samples gradient descent is able to figure out which bias(filter) is more suitable for the dataset. Then the identified bias can play as a prior knowledge to reduce the sample complexity. We provide another evidence for this:  Let the dataset contains all $e_l, l\in \iset{d}$ while $y_{e_l}=+1$ if $l$ is odd and $-1$ is $l$ is even. Model-Conv-$k$ is still able to outperform Model-$1$-Layer on this task (see Figure~\ref{fig:gcomp4}). We observe that the sign of the learned filter looks like (+, -, +, -, ...) in contrast to the ones learned in our three tasks, which are likely to be all positive or all negative. This indicates that, besides spatial shifting invariance, jointly training the convolutional filter can exploit a broader set of structures and be adaptive to different data distributions. 

\begin{figure}[ht]
\centering
\includegraphics[width=0.5\columnwidth]{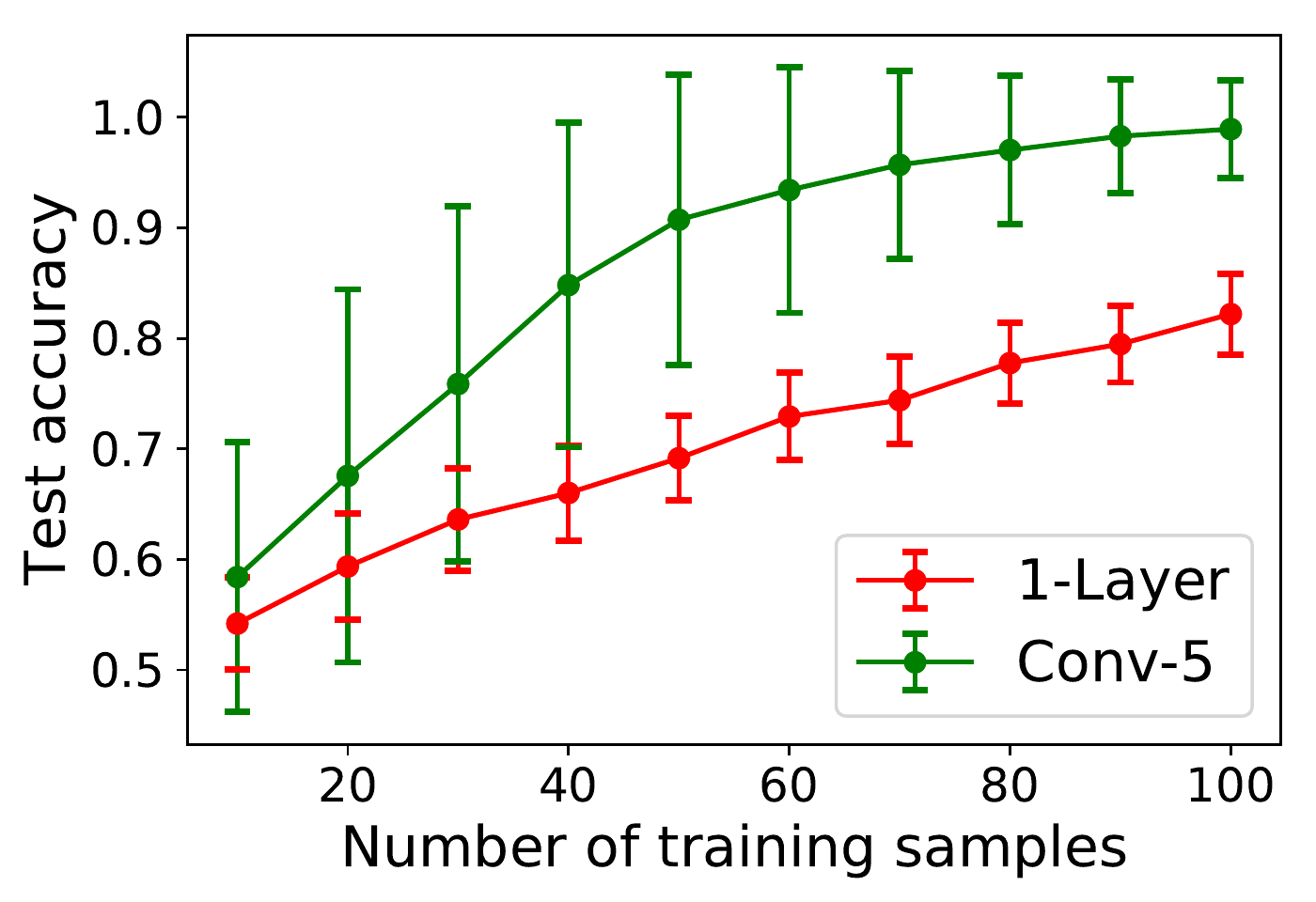}
\caption{
Classifying even v.s. odd non-zero position. Settings are the same as in Figure~\ref{fig:gcomp}. 
}
\label{fig:gcomp4}
\end{figure}

\section{Correlation Between Normal-hinge and X-hinge under Different Initializations}

Figure~\ref{fig:winit-2} and \ref{fig:winit-3} shows the variance introduced by weight initialization is also strongly correlated under two losses in Task-1stCtrl and Task-3rdCtrl. Figure~\ref{fig:winit-3-1} looks a bit different from the other two tasks because the extreme hinge loss is biased and $w^{\infty}$ may not able to separate the training samples in Task-3rdCtrl. But the strong correlation between the normal hinge loss and the extreme hinge loss under different weight initializations still holds.  

\begin{figure}[ht]
\vskip -0.1in
\centering
\subfigure[Convergence with $t$.]{
\includegraphics[width=0.45\columnwidth]{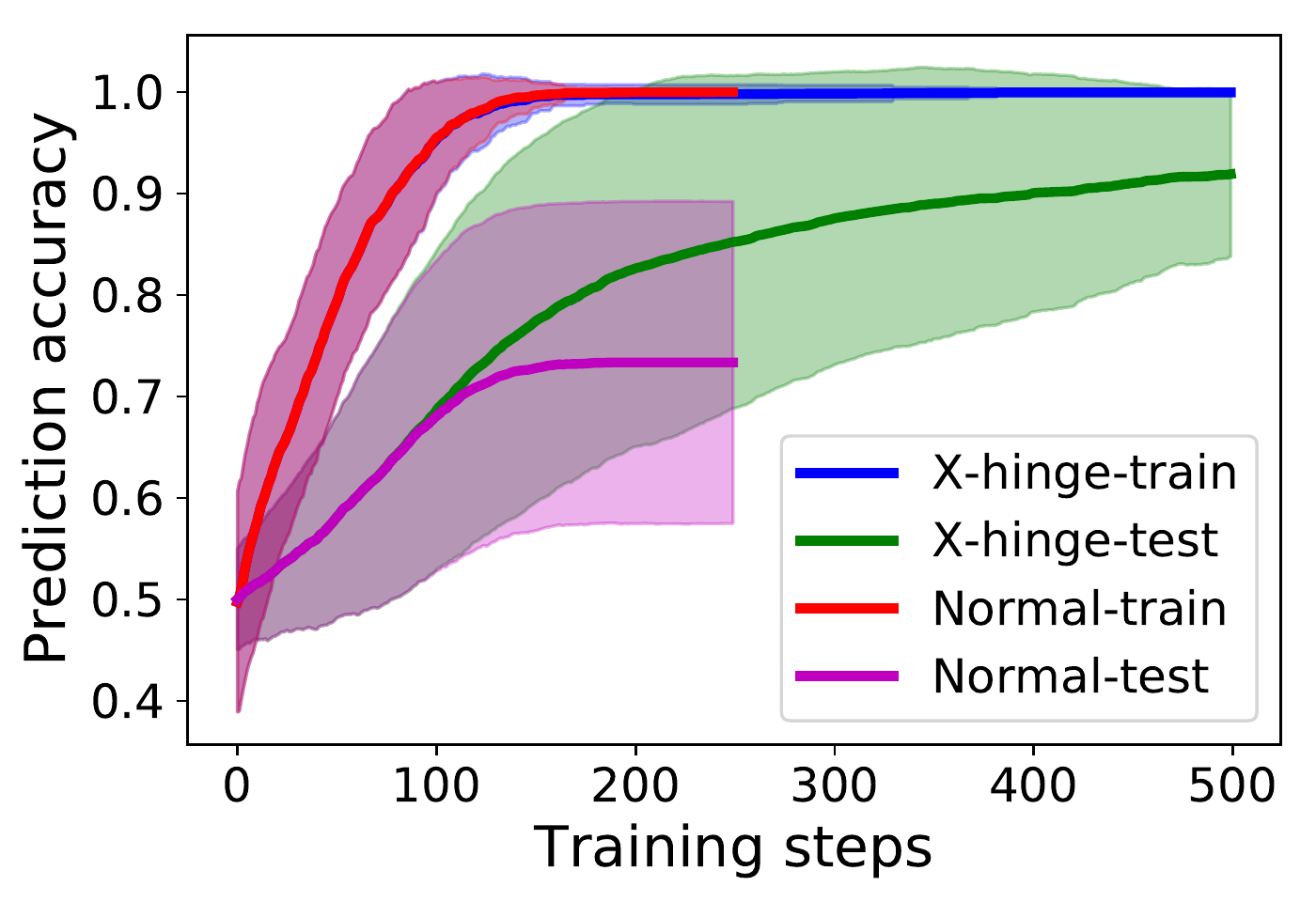}
\label{fig:winit-2-1}
}
\subfigure[Correlation at $t=150$.]{
\includegraphics[width=0.45\columnwidth]{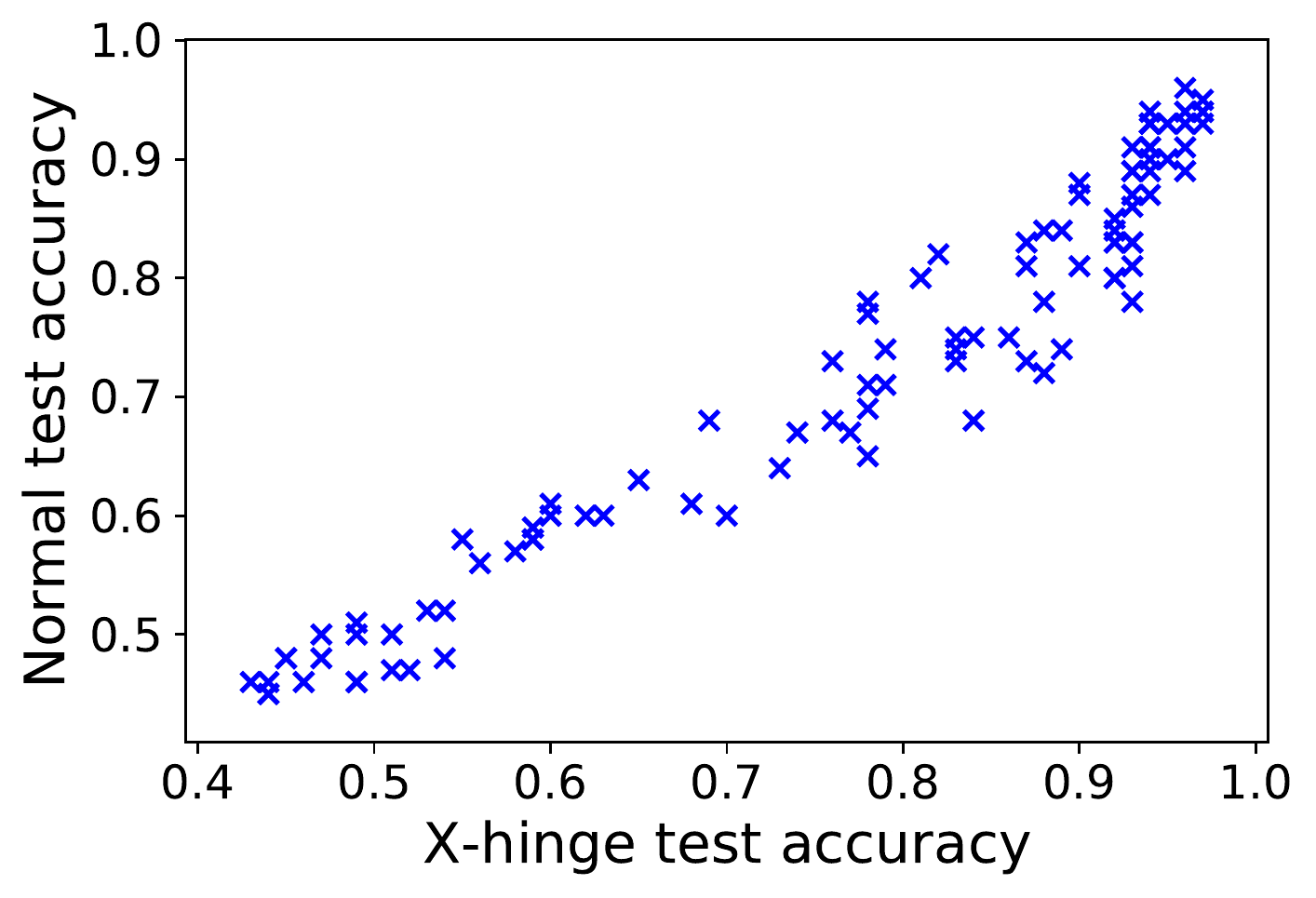}
\label{fig:winit-2-2}
}
\caption{
The effect of weight initialization in Task-1stCtrl. We fix $d=100$, $n=30$ and train Model-Conv-$k$ with $100$ different random initializations using both losses.
}
\label{fig:winit-2}
\vskip -0.1in
\end{figure}
\begin{figure}[ht]
\vskip -0.1in
\centering
\subfigure[Convergence with $t$.]{
\includegraphics[width=0.45\columnwidth]{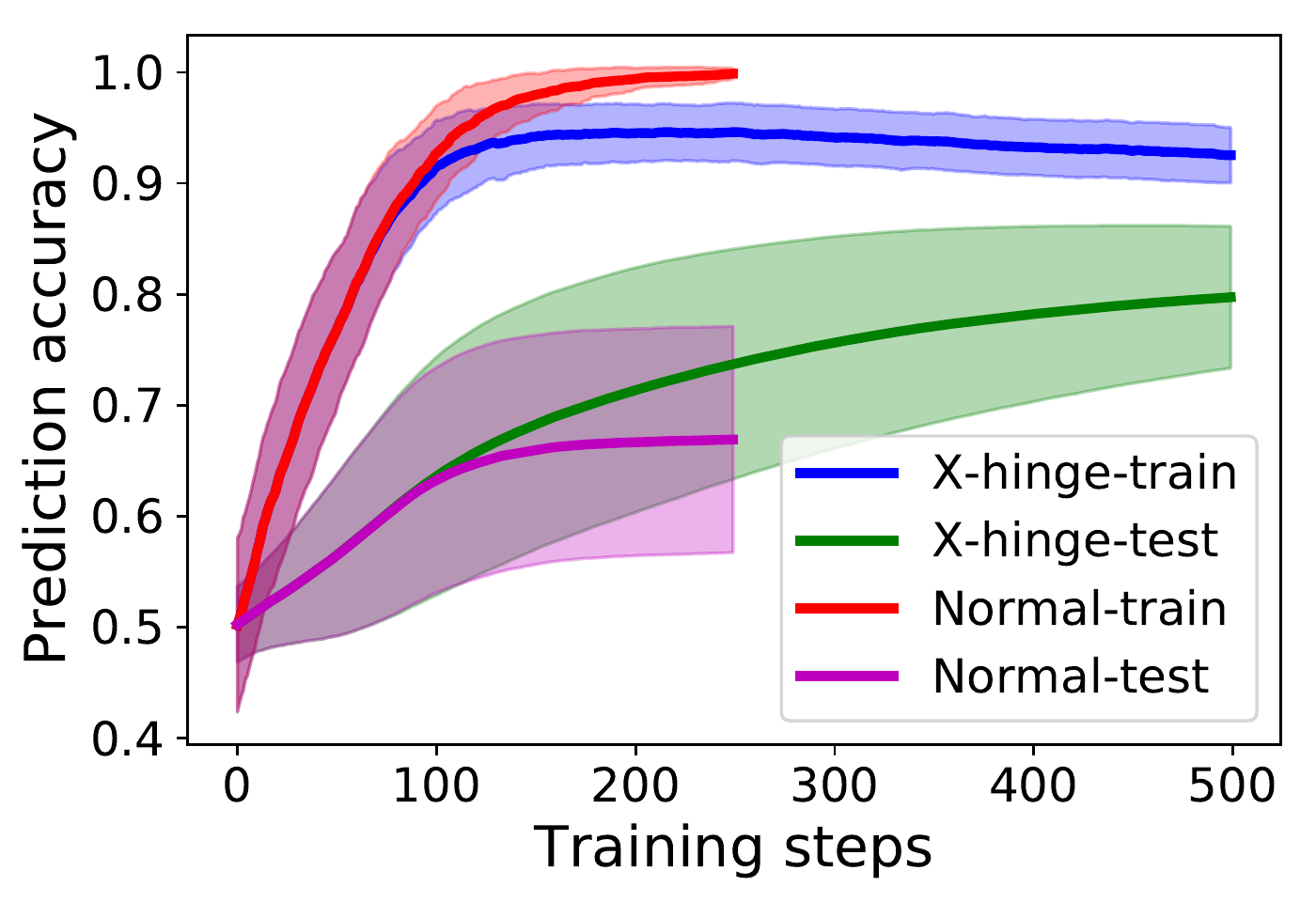}
\label{fig:winit-3-1}
}
\subfigure[Correlation at $t=150$.]{
\includegraphics[width=0.45\columnwidth]{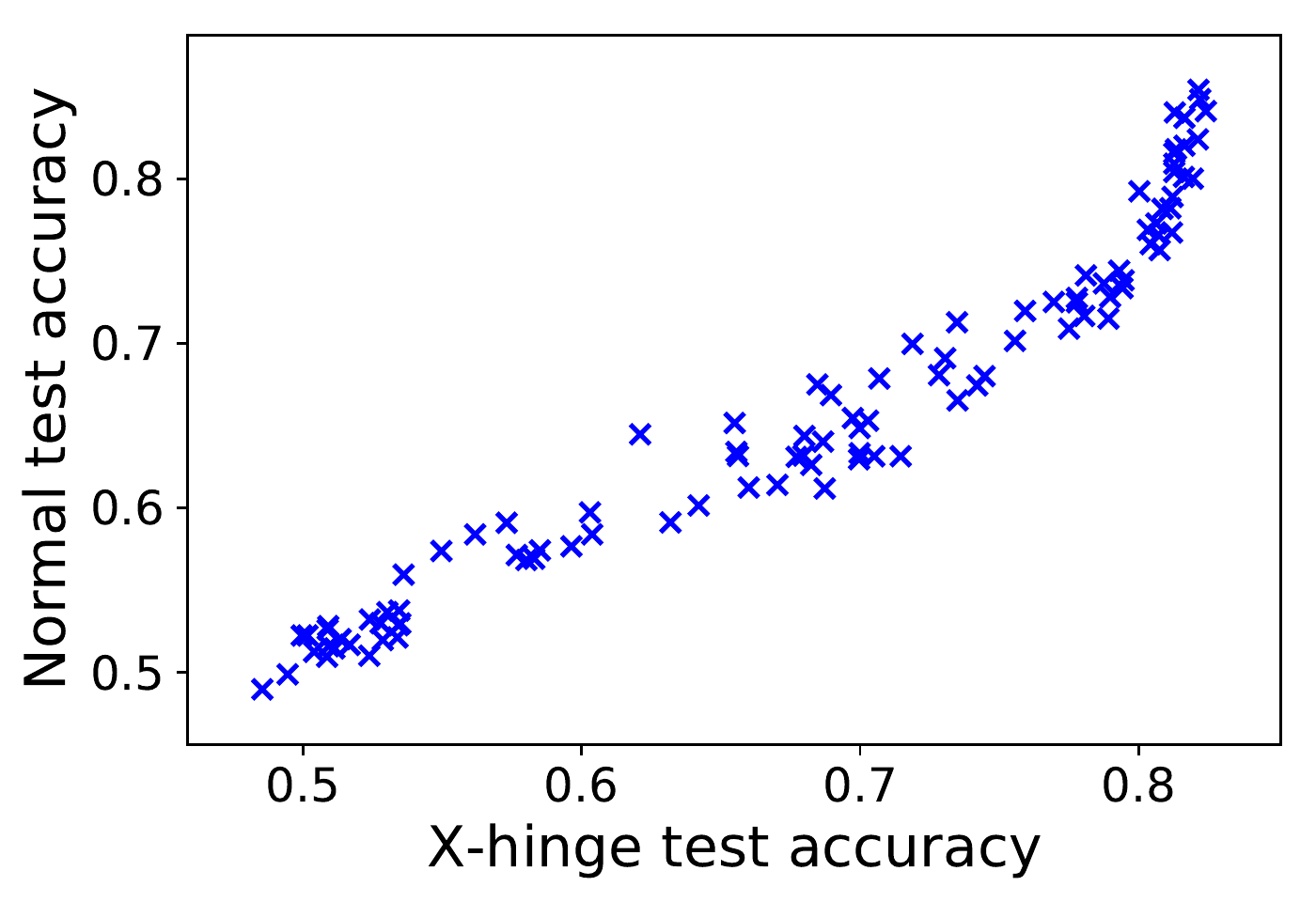}
\label{fig:winit-3-2}
}
\caption{
The effect of weight initialization in Task-3stCtrl. We fix $d=100$, $n=50$ and train Model-Conv-$k$ with $100$ different random initializations using both losses.
}
\label{fig:winit-3}
\vskip -0.1in
\end{figure}

\section{The bias of X-Hinge in Task-3rdCtrl and Potential Practical Indications}

In Figure~\ref{fig:winit-3-1} we observe that running gradient descent may not be able to achieve $0$ training error even if the samples are linearly separable. To explain this, simply consider a training set with 3 samples and $k=1, d=4$: $x_1 = [-1, 1, 0, 0], x_2 = [0, -1, 1, 0], x_3=[0,0,-1,1]$. All labels are positive. Then $\Mtr=[-1/3, 0, 0, 1/3]$. If we optimize the X-hinge loss then the network has no intent to classify $x_2$ correctly. 

Notice that in Figure~\ref{fig:winit-3-1}, under X-hinge, the generalization performance is still improving even after the training accuracy starts to decrease.  We conjecture that this indicates a new way of interpreting the role of regularization in deep nets. On real datasets we typically use sigmoid with cross entropy loss which can be viewed and a smoothed version of the hinge loss. We say a data sample is \emph{active} during training if $yf(x)$ is small so that the gradient for fitting $(x, y)$ is salient since it is not well fit yet. With X-hinge all samples are ``equality active". One message delivered by our observation is that having more samples to be ``active'' during training will make convolution filters have better generalization property, but may hurt with training data fitting. In practice we cannot recommend using X-hinge loss since the network will fail to fit the training set if we keep \emph{all} samples to be equally ``active''. But we can view this as a trade off when using logistic loss: keeping more samples to be ``active" during training with gradient descent will help with some generalization property (e.g. better Conv filters) but cause underfitting. For regularization we may want to keep as many samples to be active as possible while still be able to fit the training samples. This provides a new view of the role of regularization: Taking weight norm regularization as an example, traditional interpretation is that controlling the weight norm will reduce the capacity of neural nets, which may not be sufficient to explain non-overfitting in very large nets. The new potential interpretation is that, if we keep the weight norm to be small during training, the training samples are more ``active" during gradient descent so that better convolution filters can be learned for generalization purposes.  Verifying this conjecture on real datasets will be an interesting future direction.

\end{document}